\documentclass[accepted]{uai2025} 
                        
\usepackage[american]{babel}

\usepackage{natbib} 
\usepackage{mathtools} 
\usepackage{booktabs} 
\usepackage{tikz} 

\usepackage[utf8]{inputenc}
\usepackage{enumitem}  
\usepackage{bm,amsfonts,amsmath, mathtools, dsfont}
\usepackage{comment}
\usepackage{hyperref}
\usepackage{bbm}
\usepackage{subfig}
\usepackage{adjustbox}
\usepackage{amsthm}
\usepackage{makecell}
\usepackage{hyperref}
\usepackage{float}

\usepackage{color}

\usepackage[ruled,vlined]{algorithm2e}

\newtheorem{thm}{Theorem}
\newtheorem{Assumption}{Assumption}

\usepackage{amsmath, amssymb}

\definecolor{forestgreen}{rgb}{0.0, 0.27, 0.13}

\tikzset{every picture/.style={line width=0.75pt}} 

\newcommand{\x}{\mathbf{x}}

\newcommand{\X}{\mathbf{X}}

\renewcommand{\P}{\mathbb{P}}   
\newcommand{\I}{\mathbb{I}}
\newcommand{\E}{\mathbb{E}}
\newcommand{\V}{\mathbb{V}}

\newcommand{\Ind}{\mathbf{1}}

\newcommand{\train}{\textrm{train}}
\newcommand{\calib}{\textrm{cal}}

\newcommand{\ourmethod}{\texttt{EPICSCORE}}
\newcommand{\ourmethodshort}{\textrm{EPIC}}

\title{Epistemic Uncertainty in Conformal Scores: A Unified Approach}

\author[1, 2]{\href{mailto:<lucruz45.cab@gmail.com>?Subject=Your UAI 2025 paper}{Luben~M.~C.~Cabezas}{}}
\author[1]{Vagner~S.~Santos}
\author[1]{Thiago~R.~Ramos}
\author[1]{Rafael~Izbicki}
\affil[1]{%
    Department of Statistics\\
    Federal University of S\~ao Carlos\\
    S\~ao Carlos, S\~ao Paulo, Brazil
}
\affil[2]{%
    Institute of Mathematics and Computer Science\\
    University of S\~ao Paulo\\
    S\~ao Carlos, S\~ao Paulo, Brazil
}
  
\begin{document}
\maketitle

\begin{abstract}
Conformal prediction methods create prediction bands with distribution-free guarantees but do not explicitly capture epistemic uncertainty, which can lead to overconfident predictions in data-sparse regions.
Although recent conformal scores have been developed to address this limitation, they are typically designed for specific tasks, such as regression or quantile regression. Moreover, they rely on particular modeling choices for epistemic uncertainty, restricting their applicability. We introduce \ourmethod, a model-agnostic approach that enhances any conformal score by explicitly integrating epistemic uncertainty. Leveraging Bayesian techniques such as Gaussian Processes, Monte Carlo Dropout, or Bayesian Additive Regression Trees, \ourmethod \  adaptively expands predictive intervals in regions with limited data while maintaining compact intervals where data is abundant. As with any conformal method, it preserves finite-sample marginal coverage.
Additionally, it also achieves asymptotic conditional coverage. Experiments demonstrate its good performance compared to existing methods. Designed for compatibility with any Bayesian model, but equipped with distribution-free guarantees, \ourmethod\  provides a general-purpose framework for uncertainty quantification in prediction problems.
\end{abstract}

\section{Introduction}\label{sec:intro}

 Machine learning models traditionally focus on point predictions, estimating target variables from input features. However, understanding prediction uncertainty is crucial in many applications \citep{horta2015potential, freeman2017unified,izbicki2017converting,mcallister2017concrete,Schmidt2020Photo-z,selvan2020uncertainty,vazquez2022conformal, csillag2023amnioml,mian2024literature, Valle2024, frohlich2025personalizedus}. This has led to increased interest in uncertainty quantification methods, particularly conformal prediction, which constructs predictive regions with finite-sample validity under mild i.i.d. assumptions \citep{vovk2005algorithmic,shafer2008tutorial}. Unlike probabilistic models that rely on asymptotic assumptions or priors, conformal methods provide a distribution-free framework with guaranteed coverage.
 
Conformal prediction works by designing a non-conformity score, $s(\x,y)$, which measures the degree to which a given label value $y$ aligns with the feature values $\x$ of an instance.  Given a new input $\x_{\text{new}}$, the method constructs a predictive region by inverting the non-conformity score at a given confidence level (see Section \ref{sec:review_conformal}). The choice of $s(\x,y)$ is critical, as it directly influences the shape and informativeness of the resulting predictive regions \citep{angelopoulos2021gentle,Izbicki2025}. For instance, in regression problems, a standard choice is $s(\x,y)=|y-g(\x)|$, where $g(\x)$ is a point prediction for $Y$, typically an estimate of the regression function $\E[Y|\x]$ \citep{Lei2018}. Another common option is $s(\x,y)=\max\{ \widehat q_{\alpha_1}(\x)-y,y- \widehat q_{\alpha_2}(\x)\}$,  where $\widehat q_{\alpha_1},\widehat q_{\alpha_2}$  are quantile estimates of $Y|\x$ \citep{romano2019}.   

Despite offering distribution-free guarantees, standard conformal scores only capture \emph{aleatoric uncertainty}, which arises from inherent randomness in the data generation process - specifically, the fact that $\x$ does not uniquely determine $y$. For example, in the cases discussed above, both $E[Y|\x]$ and $q_\alpha(\x)$ reflect this form of uncertainty. However, an equally important source of uncertainty is \emph{epistemic uncertainty}, which stems from limitations in training data and the resulting lack of knowledge about the true data-generating process \citep{hullermeier2021aleatoric}.

Although conformal sets do take epistemic uncertainty into account \citep{angelopoulos2021gentle}
, the  conformal scores themselves typically do not. It follows that conformal inference often leads to misleadingly narrow predictive intervals in regions with little or no training data. Figure \ref{fig::reg_split} illustrates this issue in a regression problem (see Appendix \ref{sec::technicalReg} for technical details): in the range $x \in (7,8)$, there is essentially no training data, so we expect predictive regions to widen, reflecting the increased uncertainty. However, standard conformal methods (e.g., regression-split) instead produce overconfident, narrow intervals in this region. To address this limitation, we propose a novel approach that augments any conformal score with a measure of epistemic uncertainty. As shown in the figure, our method (\ourmethod) successfully expands the predictive regions where data is scarce, providing a more reliable uncertainty quantification.
\begin{figure}[h]
    \centering
    \begin{adjustbox}{width=\columnwidth, center}
        \includegraphics[width=\columnwidth]{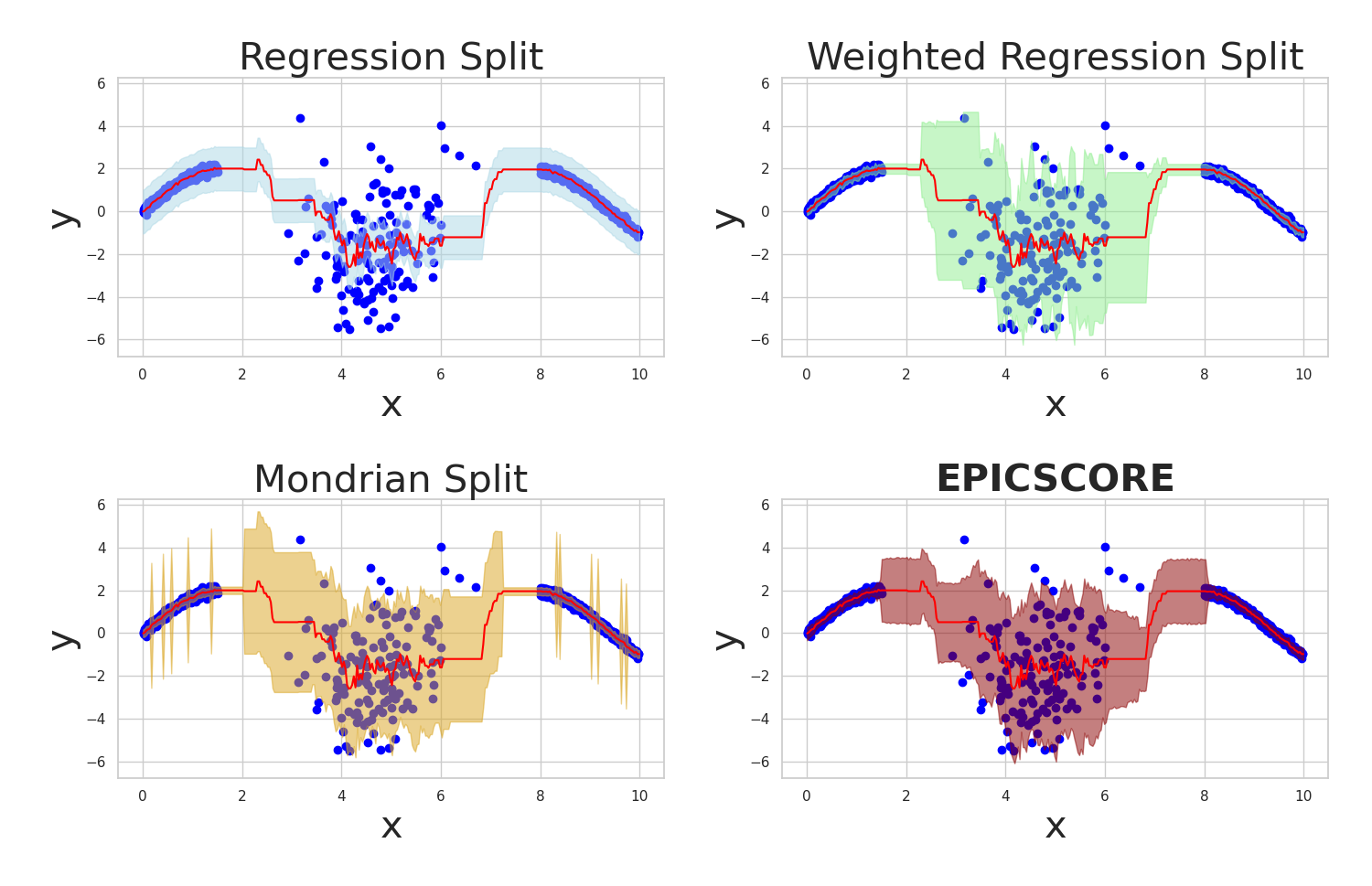}
    \end{adjustbox}
\caption{A comparison of predictive intervals from standard split-conformal regression and our proposed \ourmethod\ approach. While all methods maintain valid marginal coverage, standard conformal prediction often produces overconfidently narrow intervals in the data-scarce region (e.g. $x \in (7,8)$). Our method explicitly accounts for epistemic uncertainty, resulting in appropriately widened predictive intervals that better reflect total uncertainty when extrapolating beyond the training distribution.}
    \label{fig::reg_split}
\end{figure}

This limitation also appears in classification tasks, where conventional conformal methods produce overconfident prediction sets for test instances outside the training distribution. Using a ResNet-34 pre-trained on ImageNet (see Appendix \ref{sec::technicalClass} for details and additional results), Figure \ref{fig::images} compares Adaptive Prediction Sets (APS, \citealt{romano2020classification}) and \ourmethod\ on CIFAR-100 images. Both maintain valid coverage, but \ourmethod\ explicitly quantifies epistemic uncertainty: prediction sets expand for outliers while remaining concise for in-distribution examples. \ourmethod\ also shows improved Size-Stratified Coverage (SSC; \citealt{angelopoulos2021gentle}) by 33\% over APS, thus producing better prediction sets.  

\begin{figure*}[ht]
    \centering
        \includegraphics[width = 0.8\textwidth]{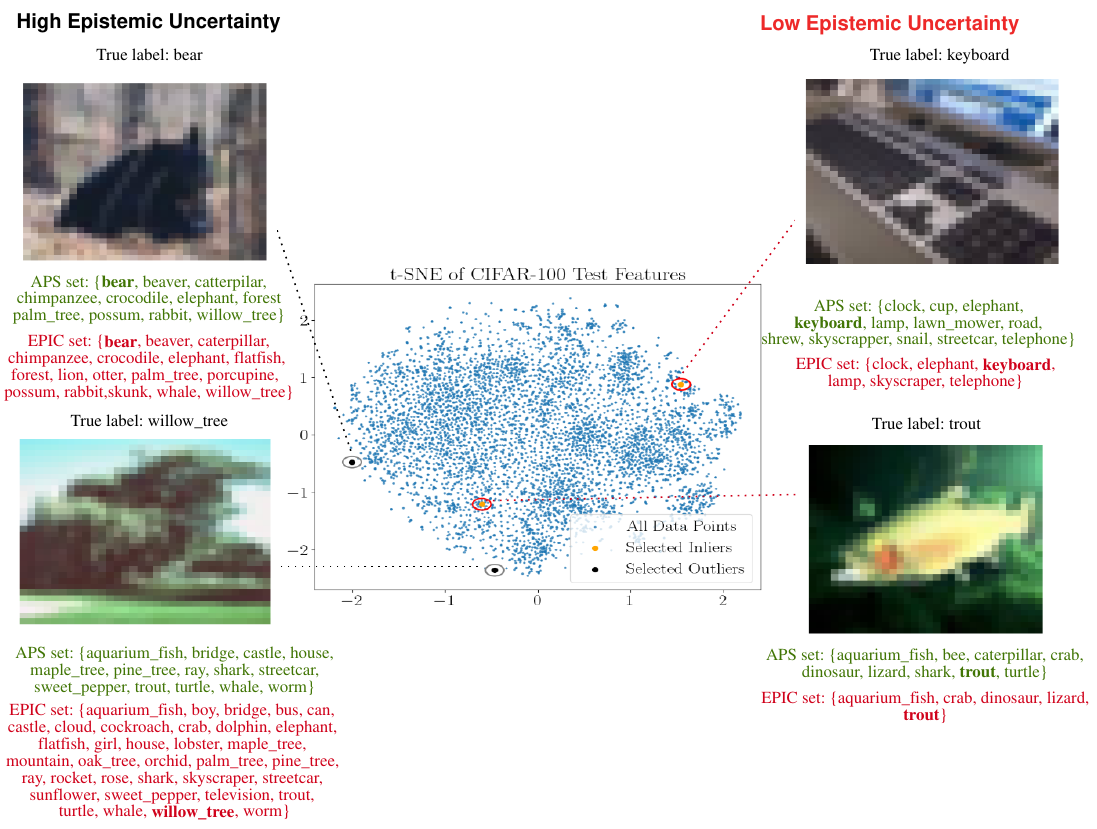}
\caption{Prediction sets from Adaptive Prediction Sets (APS) versus the proposed \ourmethod\ approach on CIFAR-100 images. Both methods maintain valid coverage, but \ourmethod\ explicitly quantifies epistemic uncertainty, resulting in adaptively expanded prediction sets for outlier images (e.g., those in data-sparse regions) while remaining concise for in-distribution examples. }
    \label{fig::images}
\end{figure*}

\subsection{Novelty}

We introduce \ourmethod \ (Epistemic Conformal Score), a novel nonconformity score that explicitly integrates epistemic uncertainty into the conformal prediction framework. Our key innovation is modeling the epistemic uncertainty of any given conformal score $s(\x,y)$ using a Bayesian process. This allows us to refine uncertainty quantification, particularly in regions of the feature space where data are sparse.

\ourmethod\ is  flexible and can leverage various Bayesian approaches, including Gaussian Processes, Bayesian Additive Regression Trees, or even approximations such as Neural Networks with Monte Carlo Dropout. By incorporating these probabilistic models, our method dynamically adjusts predictive regions to account for both aleatoric and epistemic uncertainty.

Despite \ourmethod\ using Bayesian models, it preserves the finite-sample validity guarantees of conformal prediction and also achieves asymptotic conditional coverage, a property that many existing conformal approaches lack. Furthermore, \ourmethod\  can be applied on top of any existing conformal score, making it an enhancement rather than a replacement.

\subsection{Relation to Other Work}

Several recent frameworks have attempted to model epistemic uncertainty in machine learning predictions (see, e.g., \citealt{he2023survey,tyralis2024review,wang2025aleatoric} and references therein). However, these methods generally lack guarantees on the coverage of their estimates.

In contrast, conformal prediction offers valid coverage guarantees under the relatively weak assumption of i.i.d. data. Since its introduction \citep{vovk2005algorithmic,shafer2008tutorial}, the method has advanced significantly in both theoretical foundations and practical applications \citep{Lei2018,angelopoulos2021gentle,fontana2023conformal,Manokhin2024}. A key challenge in this area has been enhancing the quality of predictive regions, particularly by achieving conditional coverage, which ensures validity conditional on specific feature values.

Since exact conditional coverage is unattainable without strong assumptions \citep{lei2014distribution}, research has focused on two strategies: (1) locally tuning cutoffs to adapt to the data distribution \citep{bostrom2020mondrian,foygel2021limits,guan2023localized, cabezas2024distribution,cabezas2025regression}, and (2) designing conformal scores that achieve conditional coverage asymptotically \citep{romano2019,izbicki2020flexible,chernozhukov2021distributional,izbicki2022cd,plassier2024conditionally}. Among these, 
the score introduced by \citet[Eq.~14]{dheur2025multi} is particularly relevant to our approach.
 This method
   transforms a nonconformity score $s$ via its estimated  cumulative distribution function, $s'(\x;y) = \widehat{F}(s(\x,y) | \x)$, improving conditional coverage. While $s'$ ensures asymptotic conditional coverage, it does not explicitly model epistemic uncertainty. Our approach builds on this idea by incorporating epistemic uncertainty in $s'$. In particular, we show that as the calibration sample grows, \ourmethod\ achieves asymptotic conditional coverage (Theorem~\ref{thm::conditional_coverage}).

As in \ourmethod, Bayesian ideas are often used in non-conformal frequentist inference, including what \citet{wasserman2011frasian}  terms ``Frasian inference''. Similar ideas appear in I. J. Good’s (\citealt{good1992bayes} and others) advocacy for a Bayes/non-Bayes compromise, in recent work on likelihood-free frequentist inference \citep{masserano2023simulator,dalmasso2024likelihood} and Bayes-optimal sets with frequentist coverage control \citep{hoff2023bayes}.   While \ourmethod's base model is Bayesian, the final conformal regions remain distribution-free and retain their frequentist distribution--free coverage guarantees, whichever the prior is.

Bayesian methods have also been explored as a means to incorporate epistemic uncertainty into conformal prediction. One approach is to use Bayesian predictive sets and subsequently apply conformal methods to adjust their marginal coverage. However, existing techniques typically do not leverage existing conformal scores explicitly, and do not lead to asymptotic conditional coverage  \citep{vovk2005algorithmic,wasserman2011frasian,fong2021conformal}.

Other recent studies have attempted to incorporate epistemic uncertainty directly into existing conformal scores. 
\citet{cocheteux2025uncertainty} modifies the Weighted regression-split nonconformity score \citep{Lei2018},   $s(\x,y) = |y - g(\x)| / \sigma(\x)$, by redefining 
$\sigma(\x)$ to capture epistemic uncertainty about $Y$, rather than aleatoric uncertainty only as in the original formulation. This uncertainty 
is estimated via Monte Carlo Dropout, which we also employ in some experiments. However, our approach is more flexible, as it accommodates 
any nonconformity score and, unlike this method, ensures asymptotic conditional coverage.

Another approach is proposed by \citet{rossellini2024integrating}, who modify the conformal 
quantile regression (CQR; \citealt{romano2019}) score function, $s(\x,y) = \max\{\widehat{q}_{\alpha_1}(\x) - y, y - \widehat{q}_{\alpha_2}(\x)\}$, to
account for epistemic uncertainty in the quantile estimates. While this adaptation shows promising results, it is restricted to quantile regression 
models and does not generalize to other conformal scores. Moreover, it necessarily measures epistemic uncertainty using ensembles.

Other efforts to model epistemic uncertainty include \citet{jaber2024conformal} and \citet{pion2024gaussian}, which 
integrate conformal methods with Gaussian processes. Although these approaches yield calibrated sets for Gaussian processes, 
they are  tailored to this class of models and cannot be applied to other frameworks, such as Bayesian Additive Regression 
Trees \citep{chipman2010bart}. In contrast, \ourmethod\ is model-agnostic (in the sense that it can be used on top of any nonconformity score) and can be applied to any Bayesian model for epistemic uncertainty.

\section{Methodology}\label{sec:methodology}

\subsection{Review of Conformal Prediction}
\label{sec:review_conformal}

Conformal prediction constructs valid prediction regions $R(\X_{n+1})$ under minimal assumptions. A widely used approach is split conformal prediction \citep{papadopoulos2002inductive,Lei2018}, which partitions the data into a training set $\mathcal{D}_\train$ and a calibration set $\mathcal{D}_\calib=\{(\X_1,Y_1), \dots, (\X_n,Y_n)\}$.  The training set is used to fit a nonconformity score $s(\x,y)$ such as those described in the introduction.
The conformal prediction region is given by
$$
R(\x_{n+1}) = \{ y : s(\x_{n+1}, y) \leq t_{1-\alpha} \}.
$$
The value $t_{1-\alpha}$ is set using the calibration set. Concretely,
$t_{1-\alpha}$ is set to be the $(1-\alpha)$-quantile of the calibration scores,
\[
t_{1-\alpha} = \text{Quantile}_{1-\alpha} \{ s(\X_i, Y_i) : (\X_i, Y_i) \in \mathcal{D}_\calib \}.
\]
This construction  ensures that the prediction set for a new observation $(\X_{n+1}, Y_{n+1})$ satisfies marginal coverage:
\[
\mathbb{P} (Y_{n+1} \in R(\X_{n+1})) \geq 1 - \alpha.
\]
 
 The choice of $s$ is critical in determining the shape and other properties of $R$. In the next section, we introduce a new conformal score that measures the epistemic uncertainty around any given score $s$. 

\subsection{Our Approach - \ourmethod}

We assume that a nonconformity score $s(\x, y)$ is already defined based on the training set.
Our starting point is to define a family of distributions
 that model the aleatoric uncertainty of $s(\X,Y)$ given $\X$, which is a set of distributions indexed by a parameter $\theta \in \Theta$. We denote this family by
  $\mathcal{F} = \{f(s|\x,\theta): \theta \in \Theta\}$. 
  This formulation is very general; $\Theta$ may even represent a nonparametric space.

To construct \ourmethod, we adopt a Bayesian model that places a prior distribution over $\Theta$ (or equivalently, over $\mathcal{F}$). For simplicity, we assume this prior has a density $f(\theta)$, though the method is generally applicable.  The choice of prior plays a crucial role: it can be tailored to capture epistemic uncertainty in the data-generating process (via uninformative or weakly informative priors) or to introduce stronger regularization into the conformal score distribution (via more concentrated priors). In this work, we focus on enhancing the representation of epistemic uncertainty and therefore adopt non-informative priors. However, a balance between capturing uncertainty and enforcing regularization can be explored through prior specification (see Appendix \ref{sec::prior_comparisson} for a discussion on this trade-off).

In our experiments, we use Gaussian Processes (GP), Bayesian Additive Regression Trees (BART), and approximate Bayesian Mixture Density Networks (MDNs) with Monte Carlo dropout. However, \ourmethod{} is compatible with any Bayesian modeling approach that yields a posterior predictive distribution.  These three models were chosen for their complementary strengths: BART offers flexibility and interpretability, naturally handling categorical features; GPs provide expressive nonparametric inference with scalable tractability via variational approximations; and MDNs with Monte Carlo dropout are highly scalable, making them well-suited for large-scale or high-dimensional problems.

We update the prior distribution \( f(\theta) \) using a subset of the calibration set \(\mathcal{D}_\calib\). Specifically, \(\mathcal{D}_\calib\) is split into two disjoint subsets: \(\mathcal{D}_{\calib, 1}\) and \(\mathcal{D}_{\calib, 2}\).  The first subset, \(\mathcal{D}_{\calib, 1}\) is transformed into the  dataset 
\[
D = \{(\mathbf{X}, S) : (\mathbf{X}, Y) \in \mathcal{D}_{\calib, 1}, \, S = s(\mathbf{X}, Y)\}.
\]
Using this transformed dataset, we compute the posterior distribution $f(\theta|D)$,   which reflects the updated epistemic uncertainty about the data-generating process after observing this data.
Our Bayesian model assumes that, given $\theta$, the data points $(\X,S)$ are independent and share the same conditional distribution $f(s|\x,\theta)$. 
Then, we derive the predictive cumulative distribution
\begin{align}
\label{eq:predictive_conf_score}
    F(s|\x,D)=\int F(s|\x,\theta)f(\theta|D)d\theta,
\end{align}
where $F(s|\x,\theta)$ is the CDF given by model $\theta$.

The posterior $f(\theta|D)$ is not computed explicitly. Instead,  Eq.~\ref{eq:predictive_conf_score} is approximated either via Monte Carlo sampling (as with BART and MC-Dropout) or through a  variational approximation of $f(\theta|D)$ (as in the variational GP).

Finally, our modified  nonconformity score, \ourmethod, is defined
as
\begin{align}
\label{eq:ourscore}
    s'(\x,y)=F(s(\x,y)|\x,D).
\end{align}
By construction, $s'$ incorporates epistemic uncertainty into $s$ by averaging the original score distribution, $F(s|\x,\theta)$, over the posterior $f(\theta|D)$, thus propagating the uncertainty about $\theta$ throughout the model.

Once $s'$ is computed, the prediction region for a new sample point is obtained using the standard split conformal method, with $s'$ serving as the nonconformity score. Specifically, $s'$ is evaluated for every sample in the second subset, $\mathcal{D}_{\calib, 2}$. The $(1-\alpha)$-quantile of these values, denoted as \( t_{1-\alpha} \), is then used to construct the prediction region
\[
R_{\ourmethodshort}(\x_{n+1}) = \{y: s'(\x_{n+1},y) \leq t_{1-\alpha} \},
\]
which, by the definition of $s'$, can be expressed in terms of the original nonconformity score $s$ as
\[
R_{\ourmethodshort}(\x_{n+1}) = \{y: s(\x_{n+1},y) \leq F^{-1}(t_{1-\alpha}|\x_{n+1},D) \}.
\] 


\ourmethod{} is summarized in Algorithm~\ref{alg:epicscore} and  illustrated in Figure~\ref{fig:EPICSCORE_scheme}. Additional details on prior specification and strategies for managing its influence are provided in Appendix~\ref{sec::prior_comparisson}. A discussion of the computational complexity introduced by incorporating Bayesian models into our framework is presented in Appendix~\ref{sec::computational_complexity}.

\begin{figure*}[ht]
    \centering
    \includegraphics[width = 1.0\textwidth]{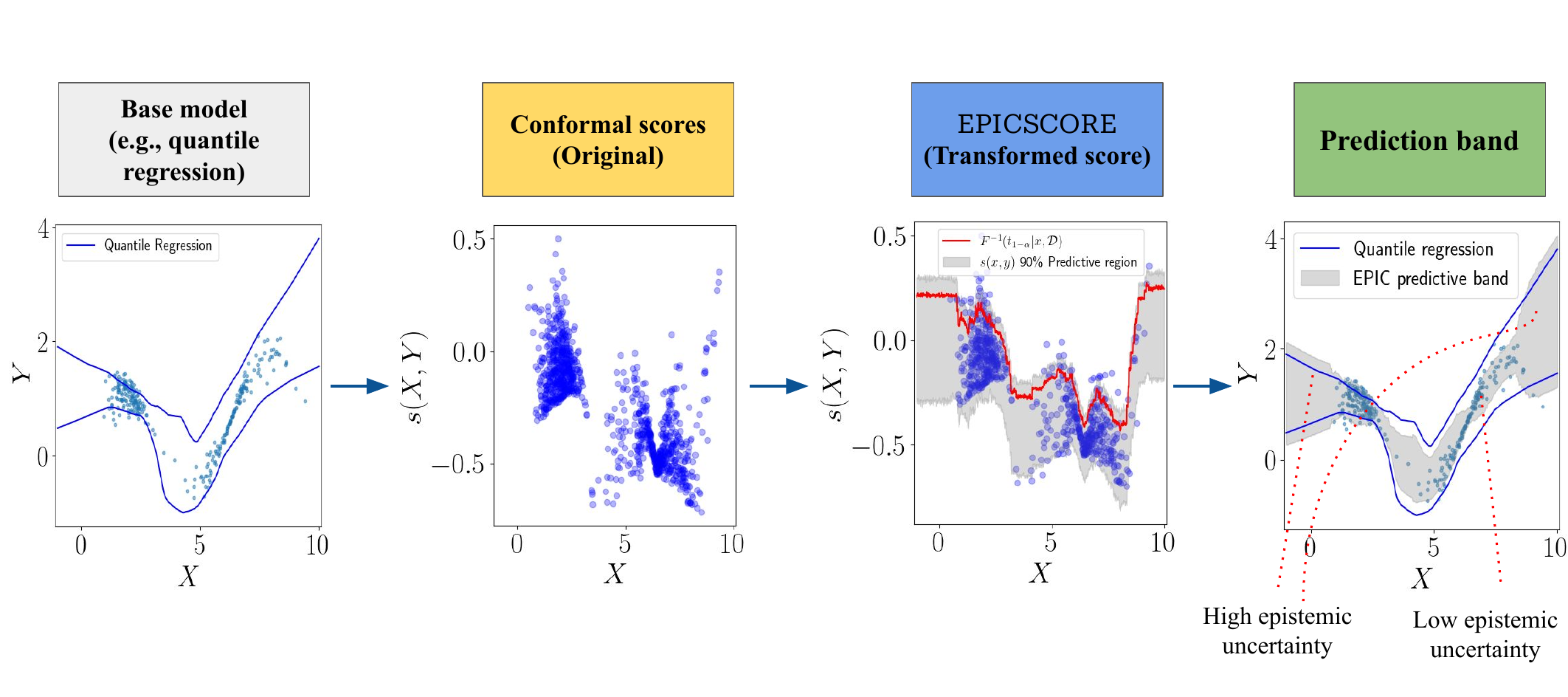}
    \caption{\ourmethod{} schematic illustration: Given a fitted base model (first panel), we begin by creating a nonconformity score and evaluating it over the calibration set (second panel). We then model the predictive distribution of the conformal score $s(\X, Y)$ using a specified family of models, integrating the epistemic uncertainty about the data-generating process. The predictive CDF of each original score defines a new conformal score, allowing threshold computation in the transformed space (third panel). Finally, leveraging these predictive-based cutoffs, we construct an adaptive prediction band that accounts for epistemic uncertainty (fourth panel).}
    \label{fig:EPICSCORE_scheme}
\end{figure*}

\begin{algorithm}[ht]
\caption{\ourmethod}
\label{alg:epicscore}
\KwIn{Data $\mathcal{D} = \{(\X_i,Y_i)\}_{i =1}^n$, conformal score $s(\X, Y)$, nominal level $\alpha$, test point $\X_{n + 1}$}

\textbf{Step I: Fit conformal scores} \\
1: Split data $\mathcal{D}$ into a training set $\mathcal{D}_{\text{train}}$ and a calibration set $\mathcal{D}_{\text{cal}}$. \\
2: Fit the conformal score $s(\X, Y)$ in $\mathcal{D}_{\text{train}}$.

\textbf{Step II: Fit the predictive function}\\
1: Split data $\mathcal{D}_{\text{cal}}$ into a training set $\mathcal{D}_{\text{cal},1}$ and a calibration set $\mathcal{D}_{\text{cal}, 2}$. \\
2. Fit predictive CDF $F(s|\x,D)$ using $\mathcal{D}_{\text{cal},1}$. \\
3. Compute \ourmethod{} conformal score $s'(\x,y)$ for all elements of $\mathcal{D}_{\text{cal}, 2}$ (Eq.~\eqref{eq:ourscore}) \\
4: Compute the $(1 - \alpha)$ empirical quantile $t_{1-\alpha}$ of the conformal scores. \\

\textbf{Step III: Compute prediction set}\\
3: Compute the set ${R}_{\ourmethodshort}(\X_{n + 1})$ as:
   \begin{align*}
   {R}_{\ourmethodshort}(&\X_{n + 1}) = \{y \mid s'(\X_{n + 1}, y) \leq t_{1-\alpha} \} \\
   &= \{y \mid s(\X_{n + 1},y) \leq F^{-1}(t_{1- \alpha}|\X_{n + 1}, D) \}
   \end{align*}
\end{algorithm}

\subsubsection{Special Cases}
\label{sec::special_cases}
We examine specific instances of conformal scores to provide further insight into how \ourmethod\ captures epistemic uncertainty.

\textbf{Regression}. 
If the original conformal score is $s(\x,y)=|y-g(\x)|$, 
the prediction regions given by \ourmethod\ have the form
\[ g(\x_{n+1}) \pm F^{-1}(t_{1-\alpha}|\x_{n+1},D),\]
In particular, if $S|\x,D$ is modeled by a normal distribution with mean $\mu(\x,D)$ and standard deviation $\sigma(\x,D)$  the prediction sets will have the shape
$$\left( g(\x_{n+1})-\mu(\x_{n+1},D)\right) \pm \sigma (\x,D)\sqrt{2} \text{erf}^{-1}(2t_{1-\alpha} - 1),$$
where $\text{erf}^{-1}$ denotes the inverse error function. 
 This is equivalent to changing the original conformal score to
$|y-g(\x_{n+1})-\mu(\x_{n+1},D)|/\sigma(\x_{n+1},D)$, which is similar to the approach by \citet{cocheteux2025uncertainty}, although any process can be used to model the epistemic uncertainty in our version.

\textbf{Quantile Regression}. 
If the original conformal score \( s \) is given by Conformalized Quantile Regression (CQR) \citep{romano2019}, the prediction regions of \ourmethod\ have the form
\begin{equation*}
\begin{split}
[q_{\alpha_1}(\x_{n+1}) - F^{-1}(t_{1-\alpha}|\x_{n+1}, D), \\
q_{\alpha_2}(\x_{n+1}) + F^{-1}(t_{1-\alpha}|\x_{n+1}, D)].
\end{split}
\end{equation*}
Unlike the original CQR formulation, which expands or contracts the quantile regions $[q_{\alpha_1}(\x_{n+1}),q_{\alpha_2}(\x_{n+1})]$ by a constant factor \( t \), \ourmethod\ adjusts the regions dynamically based on the epistemic uncertainty at \( \x_{n+1} \).  
This approach is similar, but more flexible than previous methods, such as UACQR-S \citep{rossellini2024integrating}, which imposes a correction factor of the form \( t \times g(\x_{n+1}) \).  Also, any process can be used to model
$F^{-1}(t_{1-\alpha}|\x_{n+1}, D)$.

\textbf{Classification}.  
Let $s(\x,y)$ be any nonconformity score for classification. One example is the APS score
\begin{align}\label{eq:aps}s(\mathbf{x}, y) = \sum_{y' \in \mathcal{Y}: \widehat{\P}(y'|\mathbf{x}) > \widehat{\P}(y|\mathbf{x})} \widehat{\P}(y'|\mathbf{x}),
\end{align}
where $\widehat{\P}(y'|\mathbf{x})$ represents the predicted probabilities from any classifier. Another common choice is
\begin{align}\label{eq:cdsplit}s(\mathbf{x}, y) = - \widehat{\P}(y|\mathbf{x}),
\end{align} \citep{vovk2005algorithmic, valle2023quantifying}.
Since $Y$ is discrete, the score $s(\mathbf{x}, Y)$ is also discrete. Moreover, the cumulative distribution function of its predictive distribution can be computed using the predictive distribution of the labels, $\P(y|\x,D)$. In particular, \ourmethod\ is given by
\begin{align*}
    s'(\x,y)&= \P \left(s(\x,Y) \leq s(\x,y)|\x,D \right)\\
    &=\sum_{y'}\I(s(\x,y') \leq s(\x,y)) \P(y'|\x,D)\\
    &=\sum_{y': s(\x,y') \leq s(\x,y)} \P(y'|\x,D).
\end{align*}
This formulation reveals several key insights about \ourmethod\ for classification:
\begin{itemize}

\item If the initial classifier $\widehat \P(y|\x)$ is a neural network trained with Monte Carlo dropout or batch normalization, an approximation to $\P(y|\x,D)$ is readily available; one only needs to use the same technique at test time.  This is because these methods provide a variational approximation of Bayesian predictive distributions \citep{gal2016dropout, teye2018bayesian}, eliminating the need to compute the predictive distribution using a separate holdout set.

\item Both scoring functions in Eqs.~\ref{eq:aps} and \ref{eq:cdsplit} lead to the same \ourmethod\ score. Additionally, \ourmethod\ follows a similar structure to APS, with the key difference being that the estimates $\widehat{\P}(y|\mathbf{x})$ on the sum of Eq.~\ref{eq:aps} are replaced by the predictive distribution $\P(y|\x,D)$.

\item For large $D$, no epistemic uncertainty remains and, therefore, \ourmethod \ converges to the populational version of the APS score.
\end{itemize}

\section{Theory}\label{sec:theory}

Just like any split-conformal method, \ourmethod\ guarantees marginal coverage as long as the data are exchangeable, regardless of the chosen Bayesian model (see Appendix \ref{sec:proofs} for all proofs):

\begin{thm}
\label{thm::marginal}
Assuming that the data are exchangeable, the confidence region constructed by \ourmethod~ satisfies marginal coverage, that is,
\[
\P\left(Y \in {R}_{\ourmethodshort}(\X) \right) \geq  1 - \alpha.
\]
Moreover, if the fitted scores follow a continuous joint distribution, the upper bound also holds:
\[
\P\left(Y \in {R}_{\ourmethodshort}(\X) \right) \leq  1 - \alpha + \frac{1}{1 + |\mathcal{D}_{\calib, 2}|}.
\]
\end{thm}

We now analyze its conditional coverage properties.

As the sample size of the calibration set used to compute the posterior $D$ increases, the distribution function $F(s(\x,y)|\x,D)$ typically converges to $S(\X,Y)|\x,\theta^*$, where $\theta^*$ denotes the true parameter value \citep{schervish2012theory, bernardo2009bayesian}. Consequently, \ourmethod\ recovers the score proposed by \citet[Eq.~14]{dheur2025multi} in the limit of large calibration samples, which is exactly when epistemic uncertainty is negligible. This score is known to control asymptotic conditional coverage. We show that our proposed score exhibits the same property.

Formally, we assume that the 
predictive distribution converges to the true distribution of the conformal score \citep{bernardo2009bayesian,schervish2012theory}:
\begin{Assumption}\label{assumption:uniform_convergence}
For any $\varepsilon > 0$, we assume uniform convergence in probability over the randomness in $D$:
\begin{align*}
    \lim_{|D| \to \infty} \P\left( \sup_{s, \x} \left|F(s \mid \x, D) - F(s \mid \x, \theta^*)\right| > \varepsilon \right) = 0.
\end{align*}
\end{Assumption}

Next, we show  that \ourmethod \ has  asymptotic conditional coverage:

\begin{thm}
Under Assumption \ref{assumption:uniform_convergence}, and assuming that the data are exchangeable, the confidence region constructed by \ourmethod~ satisfies the asymptotic conditional coverage condition, that is:
\label{thm::conditional_coverage}
\[
\lim_{|\mathcal{D}_{\calib}|\to \infty}\P\left(Y \in {R}_{\ourmethodshort}(\X) \mid \X = \x \right) = 1 - \alpha.
\]
\end{thm}
Assumption~\ref{assumption:uniform_convergence} plays a key role in Theorem~\ref{thm::conditional_coverage}. While strong, it serves as a practical condition to ensure that the predictive distribution $F(s \mid \x, D)$ converges to the true distribution $F(s \mid \x, \theta^*)$ as the sample grows, so that the estimated quantiles approximate the ideal ones. Although potentially stronger than necessary, relaxing it would require a more technical analysis, which we leave for future work. Empirical results in Section~\ref{sec:experiments} show that \ourmethod\ achieves coverage close to the nominal level, suggesting robustness even when the assumption holds only approximately.

It is worth emphasizing that the type of guarantee provided by Theorem~\ref{thm::conditional_coverage} is, in fact, the strongest achievable in this setting. As shown by~\cite{lei2014distribution}, exact conditional coverage is unattainable in general conformal frameworks.  Recent methods aim to improve local coverage~\citep{bostrom2020mondrian, foygel2021limits, cabezas2025regression} and could be combined with \ourmethod\ to improve local calibration.

\section{Experiments}\label{sec:experiments}

In this section, we evaluate our framework against state-of-the-art baselines, applying \ourmethod\ to two initial conformal scores: (i) quantile-based and (ii) regression-based. 
Each version of \ourmethod\ is compared to appropriate baselines, which are detailed in the following subsections. 
In Appendix \ref{sec::dens_epicscore} we also show that \ourmethod\ can be applied to density-based scores.

We consider three versions of \ourmethod, each using a different model for the predictive distribution (see implementation details in Appendix \ref{sec::comp_details}):  
\begin{itemize}
    \item \textbf{Bayesian Additive Regression Trees} \citep{chipman2010bart}: The BART model represents the score as a sum of regression trees:
    \begin{align*}
        s(Y, \X)|\X, \boldsymbol{\theta} \sim \phi\left( \sum_{i = 1}^m G_i(\X, T_i, M_i), \sigma \right) \; ,
    \end{align*}
    where $\phi$ denotes a probability distribution, $\sigma$ its associated scale, $G_i$ a binary tree with structure $T_i$ and leaf  values $M_i$. 
    We set $\phi$ as a Normal distribution and $\sigma$ depending on $\X$ to incorporate heteroskedasticity \citep{pratola2020heteroscedastic}.
    
    \item \textbf{Gaussian Process (GP)} \citep{williams2006gaussian, schulz2018tutorial}: For the GP regression model, we assume the score follows the form
$
        s(Y,\x) = f(\x) + \varepsilon \;,
 $ with $\varepsilon \sim N(0, \sigma_{\varepsilon}^2)$ representing independent Gaussian noise, and $f(\x) \sim GP(m(\x), k(\x, \x'))$ is a Gaussian Process with mean function $m(\x)$ and covariance function $k(\x, \x')$.  
    We adopt variational approximations to the predictive distribution \citep{salimbeni2018natural}, which offer scalability.
    \item \textbf{Mixture Density Network with MC-Dropout \citep{bishop1994mixture, gal2016dropout}}: The Mixture Density Network (MDN) models the score distribution using a weighted sum of Gaussian components:
    \begin{align*}
        f(s(y,\x)|\x) = \sum_{k = 1}^K \pi_k(\x)N(s(y,\x) |\mu_k(\x), \sigma^2_k(\x)) \; ,
    \end{align*}
    where $N(\cdot)$ denotes the normal density and $\pi_k(\cdot)$, $\mu_k(\cdot)$, $\sigma_k(\cdot)$ are all modeled by neural networks, with $\sum_{k = 1}^K \pi_k(\x) = 1$. 
    To derive a predictive distribution for the scores, we incorporate dropout at each MDN layer. By performing multiple stochastic forward passes using MC Dropout, we approximate the posterior distribution of the MDN parameters, thus  propagating uncertainty into the predictive score distribution \citep{gal2016dropout}. Appendix \ref{sec::mc_dropout_details} explains how MC Dropout yields predictive distributions.
\end{itemize}

Our comparisons are conducted using 13  datasets commonly employed for benchmarking in the conformal prediction literature: Airfoil \citep{dua2017uci}, Bike \citep{kaggle_bike_sharing_demand}, Concrete \citep{concrete_compressive_strength_165, dua2017uci}, Cycle \citep{combined_cycle_power_plant_294, dua2017uci}, Homes\cite{kaggle2016}, Electric \citep{dua2017uci}, Meps19 \citep{romano2019}, Protein \citep{physicochemical_properties_of_protein_tertiary_structure_265, dua2017uci}, Star \citep{achilles2008tennessee}, SuperConductivity \citep{superconductivty_data_464}, WEC \citep{neshat2020optimisation}, WineRed \citep{wine_quality_186}, and WineWhite \citep{wine_quality_186}. Additional details on these datasets are provided in Table \ref{tab:realdata}  of the Appendix.

 We report the average performance across 50  runs, highlighting methods that achieve statistically significant improvements based on 95\% confidence interval of each evaluation metric. In each run, we randomly partition the data into  40\% for training, 40\% for calibration, and 20\% for testing.

We use the Average Interval Score Loss (AISL) \citep{gneiting2007strictly} as our main evaluation metric, as it balances coverage and interval length, promoting narrower yet valid prediction intervals. To    evaluate how each method captures epistemic uncertainty, we also report the outlier-to-inlier interval length ratio and average coverage on outliers, which measure adaptivity in data-scarce regions. Additionally, we include (i) average interval length, (ii) marginal coverage, and (iii) the Pearson correlation between coverage and interval length, which serves as a proxy for conditional coverage quality \citep{feldman2021improving}. A full description of all metrics can be found in Appendix~\ref{sec::evaluation_metrics}, with extended empirical results presented in Appendix~\ref{sec::additional_res}.

\begin{table*}[ht]
\caption{Quantile regression AISL values for each method and dataset. The table reports the mean across 50 runs, with twice the standard deviation in brackets. Bold values indicate the best-performing method within a $95\%$ confidence interval. \ourmethod{} demonstrates strong performance across most datasets and consistently ranks among the top methods.}
\label{tab:aisl_quantile}
\centering
\begin{adjustbox}{max width=\textwidth}
\begin{tabular}{lccccccc}
\hline
\textbf{Dataset}  & \textbf{EPIC-BART}      & \textbf{EPIC-GP}        & \textbf{EPIC-MDN}       & \textbf{CQR}            & \textbf{CQR-r}          & \textbf{UACQR-P} & \textbf{UACQR-S}        \\ \hline
airfoil           & 19.361 (0.234)          & 19.704 (0.27)           & \textbf{18.799 (0.29)}  & 20.521 (0.234)          & 20.535 (0.236)          & 23.021 (0.337)   & 20.188 (0.3)            \\
bike $\times (10^1)$             & 44.722 (0.297)          & 47.818 (0.320)          & \textbf{43.858 (0.326)} & 45.628 (0.256)          & 45.638 (0.258)          & 53.413 (0.376)   & \textbf{43.815 (0.385)} \\
concrete          & \textbf{42.765 (0.723)} & 45.276 (0.764)          & 44.442 (0.8)            & 46.882 (0.681)          & 46.896 (0.683)          & 52.789 (1.097)   & 47.324 (1.349)          \\
cycle             & 34.435 (0.142)          & 35.054 (0.131)          & \textbf{34.077 (0.129)} & 39.218 (0.134)          & 39.408 (0.136)          & 43.775 (0.181)   & 35.346 (0.197)          \\
electric          & 0.099 (< 0.001)           & 0.096 (< 0.001)           & \textbf{0.082 (< 0.001)}  & 0.102 (0.001)           & 0.102 (0.001)           & 0.111 (0.001)    & 0.097 (< 0.001)           \\
homes $\times (10^5)$            & 7.739 (0.066)           & 8.098 (0.072)           & \textbf{7.225 (0.049)}  & 8.360 (0.075)           & 8.433 (0.078)           & 11.427 (0.131)   & 8.544 (0.107)           \\
meps19            & \textbf{65.085 (1.469)} & \textbf{64.907 (1.56)}  & \textbf{64.3 (1.528)}   & \textbf{64.239 (1.56)}  & \textbf{64.239 (1.56)}  & 71.015 (1.763)   & \textbf{63.737 (1.461)} \\
protein           & 17.687 (0.019)          & 18.096 (0.037)          & \textbf{17.417 (0.019)} & 17.7 (0.015)            & 17.7 (0.016)            & 18.149 (0.015)   & 17.691 (0.015)          \\
star $\times (10^1)$            & \textbf{98.466 (0.768)} & \textbf{98.033 (0.750)} & \textbf{98.725 (0.754)} & \textbf{97.770 (0.725)} & \textbf{97.791 (0.724)} & 99.782 (0.647)   & 99.809 (0.968)          \\
superconductivity & 74.37 (0.222)           & 80.278 (0.266)          & \textbf{70.212 (0.196)} & 75.496 (0.219)          & 75.508 (0.218)          & 87.929 (0.513)   & 73.971 (0.404)          \\
WEC $\times (10^5)$          &  2.925 (0.009)   & 2.665 (0.012)   & \textbf{2.374 (0.010)}  & 3.138 (0.009)         & 3.142 (0.009)           & 3.517 (0.010)   & 3.046 (0.010)           \\
winered           & \textbf{3.007 (0.058)}  & \textbf{3.009 (0.059)}  & \textbf{2.977 (0.05)}   & \textbf{2.979 (0.069)}  & \textbf{2.978 (0.069)}  & 3.059 (0.069)    & \textbf{2.999 (0.063)}  \\
winewhite         & 3.334 (0.03)            & 3.327 (0.034)           & \textbf{3.219 (0.03)}   & 3.316 (0.036)           & 3.315 (0.036)           & 3.378 (0.038)    & \textbf{3.2 (0.036)}    \\ \hline
\end{tabular}

\end{adjustbox}
\end{table*}

\begin{table*}[ht]
\caption{Regression AISL values for each method and dataset. The reported values represent the average across 50 runs, with two times the standard deviation in parentheses. Bolded values highlight the method with superior performance within a $95\%$ confidence interval. \ourmethod{} demonstrates competitive or superior performance compared to other methods.}
\label{tab:aisl_reg}
\centering
\begin{adjustbox}{max width=\textwidth}
\begin{tabular}{lcccccc}
\hline
\textbf{Dataset}  & \textbf{EPIC-BART}                & \textbf{EPIC-GP}          & \textbf{EPIC-MDN}                 & \textbf{Mondrian}       & \textbf{Reg-split}         & \textbf{Weighted}                 \\ \hline
airfoil           & \textbf{19.747 (0.767)}           & \textbf{20.287 (0.686)}   & \textbf{19.823 (0.675)}           & 21.532 (0.919)          & 21.201 (0.98)              & \textbf{20.276 (0.819)}           \\
bike $\times (10^1)$           & \textbf{36.381 (0.463)}           & 41.448 (0.575)            & \textbf{37.041 (0.452)}           & 38.190 (0.403)          & 43.918 (0.567)             & 37.773 (0.468)                    \\
concrete          & \textbf{52.098 (2.237)}           & \textbf{52.998 (2.359)}   & \textbf{51.648 (2.185)}           & 61.915 (2.815)          & \textbf{54.902 (2.634)}    & 58.399 (3.165)                    \\
cycle             & \textbf{19.418 (0.211)}           & \textbf{19.522 (0.221)}   & \textbf{19.436 (0.213)}           & \textbf{19.403 (0.226)} & \textbf{19.73 (0.208)}     & \textbf{19.49 (0.207)}            \\
electric          & \textbf{0.048 (\textless{}0.001)} & 0.049 (\textless{}0.001)  & \textbf{0.048 (\textless{}0.001)} & 0.05 (\textless{}0.001) & 0.05 (0.001)               & \textbf{0.048 (\textless{}0.001)} \\
homes $\times (10^5)$             & 5.921 (0.0716)                    & 6.192 (0.0689)            & \textbf{5.546 (0.055)}           & 5.710 (0.053)           & 7.569 (0.098)              & 5.860 (0.056)                     \\
meps19            & 86.039 (2.421)                    & 87.086 (2.405)            & \textbf{75.061 (1.807)}           & 79.192 (1.821)          & 109.83 (2.695)             & 92.433 (3.259)                    \\
protein           & 18.885 (0.054)                    & 18.772 (0.065)            & 17.735 (0.055)                    & \textbf{17.586 (0.051)} & 19.423 (0.055)             & 18.314 (0.065)                    \\
star $\times (10^1)$ & \textbf{105.616 (1.255)}        & \textbf{106.112 (0.998)} & \textbf{106.368 (1.173)}        & 109.346 (1.119)       & \textbf{105.250 (1.038)} & 129.492 (1.657)   \\
superconductivity & 54.895 (0.364)                    & 59.16 (0.449)             & \textbf{53.406 (0.365)}           & 58.065 (0.313)          & 68.183 (0.418)             & 54.981 (0.345)                    \\
WEC $\times (10^5)$ & 1.437 (0.010) & 1.435 (0.011) & \textbf{1.283 (0.009)} & 
\textbf{1.294 (0.009)} & 1.620 (0.009) & 1.410 (0.009) \\
winered           & \textbf{3.152 (0.07)}             & \textbf{3.171 (0.064)}    & \textbf{3.101 (0.062)}            & 3.262 (0.069)           & \textbf{3.214 (0.063)}     & 3.415 (0.067)                     \\
winewhite         & \textbf{3.104 (0.027)}            & 3.187 (0.029)             & \textbf{3.129 (0.029)}            & \textbf{3.087 (0.023)}  & 3.181 (0.028)              & 3.189 (0.033)                     \\ \hline
\end{tabular}
\end{adjustbox}
\end{table*}

\subsection{Quantile-Regression Baselines}
 
 For quantile regression-based scores, we adopt CatBoost \citep{dorogush2018catboost} as the base quantile-regression model in all conformal methods. See Appendix \ref{sec::base_model_details} for details on hyperparameters. 
  We compare \ourmethod\ to the following baselines:
\begin{itemize}
\item \textbf{CQR}  \citep{romano2019}, the conformal quantile regression method described in the introduction.
    \item \textbf{CQR-r}  \citep{sesia2020comparison}, which scales each derived cutoff by the interval width to produce adaptive intervals. As CQR, this approach accounts only for aleatoric uncertainty.
    \item \textbf{UACQR-P} and \textbf{UACQR-S} \citep{rossellini2024integrating}, that  integrate epistemic uncertainty  through ensemble-based statistics.We use their default strategies: UACQR-S with ensemble standard deviation, and UACQR-P with ensemble order statistics.
\end{itemize}
All baselines are fitted and evaluated using the implementation from \cite{rossellini2024integrating}.

\begin{table*}[ht]
\caption{Average outlier-to-inlier interval length ratio across methods and datasets in the quantile regression setting. This table reports the average ratio between prediction interval lengths for outliers and inliers across 50 runs. Bold entries denote the best-performing method within a 95\% confidence interval. Overall, \ourmethod{} consistently achieves higher ratios, demonstrating its effectiveness in adjusting interval widths based on data sparsity.}
\label{tab:interval_ratio_quantile} 
\centering
\begin{adjustbox}{max width=\textwidth}
\begin{tabular}{cccccccc}
\hline
\textbf{Dataset}  & \textbf{EPIC-BART}     & \textbf{EPIC-GP}       & \textbf{EPIC-MDN}      & \textbf{CQR}           & \textbf{CQR-r}         & \textbf{UACQR-P}       & \textbf{UACQR-S}       \\ \hline
airfoil           & \textbf{1.028 (0.034)} & \textbf{1.01 (0.018)}  & \textbf{1.048 (0.04)}  & \textbf{1.005 (0.013)} & \textbf{1.004 (0.012)} & \textbf{1.002 (0.01)}  & \textbf{1.005 (0.016)} \\
bike              & 0.991 (0.007)          & \textbf{1.01 (0.007)}  & 0.989 (0.007)          & 0.985 (0.007)          & 0.985 (0.007)          & 0.99 (0.005)           & 0.981 (0.01)           \\
concrete          & 0.983 (0.033)          & 0.995 (0.016)          & \textbf{1.043 (0.029)} & 0.998 (0.012)          & 0.998 (0.012)          & 0.998 (0.008)          & 0.998 (0.012)          \\
cycle             & 0.953 (0.014)          & 0.966 (0.014)          & 0.949 (0.014)          & \textbf{0.997 (0.005)} & \textbf{0.997 (0.005)} & \textbf{1.0 (0.004)}   & \textbf{0.998 (0.007)} \\
electric          & \textbf{1.01 (0.004)}  & \textbf{1.01 (0.004)}  & \textbf{1.012 (0.008)} & \textbf{1.008 (0.004)} & \textbf{1.007 (0.003)} & \textbf{1.006 (0.003)} & \textbf{1.01 (0.005)}  \\
homes             & 1.116 (0.019)          & 1.101 (0.016)          & \textbf{1.173 (0.027)} & 1.088 (0.015)          & 1.081 (0.014)          & 1.039 (0.006)          & 1.094 (0.014)          \\
meps19            & \textbf{1.052 (0.029)} & \textbf{1.042 (0.025)} & \textbf{1.044 (0.028)} & \textbf{1.04 (0.025)}  & \textbf{1.04 (0.025)}  & \textbf{1.021 (0.015)} & \textbf{1.045 (0.026)} \\
protein           & 1.001 (0.002)          & 1.001 (0.003)          & \textbf{1.006 (0.002)} & 1.002 (0.001)          & 1.002 (0.001)          & 1.001 (0.001)          & 1.002 (0.001)          \\
star              & \textbf{0.994 (0.006)} & \textbf{0.994 (0.003)} & \textbf{0.999 (0.007)} & \textbf{0.994 (0.003)} & \textbf{0.993 (0.003)} & \textbf{0.997 (0.003)} & \textbf{0.997 (0.003)} \\
superconductivity & 1.021 (0.007)          & 0.987 (0.006)          & 1.029 (0.008)          & \textbf{1.041 (0.007)} & \textbf{1.041 (0.007)} & 1.023 (0.004)          & \textbf{1.049 (0.008)} \\
WEC               & \textbf{1.012 (0.006)} & 1.003 (0.008)          & \textbf{1.02 (0.01)}   & \textbf{1.016 (0.003)} & \textbf{1.016 (0.003)} & 1.011 (0.002)          & \textbf{1.021 (0.004)} \\
winered           & 1.008 (0.012)          & 1.007 (0.008)          & \textbf{1.072 (0.022)} & 0.993 (0.008)          & 0.992 (0.008)          & 0.996 (0.005)          & 0.992 (0.009)          \\
winewhite         & 1.031 (0.007)          & 1.024 (0.006)          & \textbf{1.053 (0.014)} & 1.018 (0.007)          & 1.018 (0.007)          & 1.018 (0.006)          & \textbf{1.034 (0.01)}  \\ \hline
\end{tabular}
\end{adjustbox}
\end{table*}

\subsection{Regression baselines}

 For regression-based conformal scores, we use a neural network optimized with a penalized Mean Squared Loss.  Detailed descriptions of the architectures and hyperparameters used can be found in Appendix  \ref{sec::base_model_details}.
We compare \ourmethod\ to the following baselines: \begin{itemize} 
\item \textbf{Regression Split} \citep{lei2014distribution}, the conformal method based on residuals from a regression model described in the introduction. 
\item \textbf{Weighted Regression Split} \citep{Lei2018}, which multiplies the derived cutoff by a conditional Mean Absolute Deviance (MAD) estimate to yield adaptive intervals. The MAD is modeled by regressing the training set's absolute residuals on $\X$, using the same model architecture as the base predictor.
\item \textbf{Mondrian Conformal Regression} \citep{bostrom2020mondrian}, which enhances conditional coverage by adaptively partitioning the feature space using a binning scheme based on conditional variance. We estimate variance by fitting a Random Forest to $(\X, Y)$.
\end{itemize}

\subsection{Results}

The mean average coverage is close to the nominal 90\% for all methods (Table \ref{tab:amc} and \ref{tab:amc_reg} of Appendix \ref{sec:additional}),  which is expected since all methods are conformal.

In the quantile regression setting, Table~\ref{tab:aisl_quantile} highlights \ourmethod{}'s strong performance across all datasets, with the MDN-MC Dropout variant ranking among the top in 12 out of 13 cases. Notably, \ourmethod{} outperforms all competing methods on 7 datasets, effectively balancing coverage and interval sharpness. Table~\ref{tab:interval_ratio_quantile} demonstrates \ourmethod{}'s superior adaptability to data-sparse regions, with the MDN-MC Dropout variant achieving the highest outlier-to-inlier interval length ratio in 10 out of 13 cases. Table~\ref{tab:coverage_outlier} further confirms good coverage for outliers in those same sparse regions. Additionally, Tables~\ref{tab:il} and~\ref{tab:pcorr} (Appendix~\ref{sec:additional}) show that this variant  yields narrower intervals while maintaining strong conditional coverage.

For the regression setting, Table~\ref{tab:aisl_reg} confirms \ourmethod{}'s strong performance across most datasets, with the MDN-MC Dropout variant excelling in 12 out of 13 cases and the BART version ranking among the top in 8. Tables~\ref{tab:interval_ratio_reg} and~\ref{tab:coverage_outlier_reg} report the outlier-to-inlier interval length ratio and outlier coverage, respectively—both showing that the MDN-MC Dropout variant performs consistently well in data-sparse regions. Additionally, Tables~\ref{tab:il_reg} and~\ref{tab:pcorr_reg} (Appendix~\ref{sec:additional}) confirm its ability to generate concise intervals while preserving approximate conditional coverage. 

\section{Final Remarks}\label{sec:final}

We introduce \ourmethod, a novel conformal score that incorporates epistemic uncertainty into predictive regions. Using Bayesian modeling, \ourmethod\ dynamically adjusts any nonconformity score to account for epistemic uncertainty, ensuring coverage even in sparse regions. We prove it preserves marginal coverage and achieves asymptotic conditional coverage.
Empirical results show \ourmethod\ often outperforms alternatives, producing prediction regions that better reflect uncertainty while maintaining valid coverage.

Unlike previous approaches that rely on specific modeling choices or task-dependent formulations, \ourmethod\ is fully model-agnostic. Any Bayesian model can estimate the epistemic uncertainty of a given conformal score, allowing practitioners to tailor the method to their application. This flexibility extends \ourmethod's applicability across regression, classification, and structured prediction problems.

Looking ahead, we plan to extend \ourmethod{} to settings involving distribution shift, where accounting for epistemic uncertainty is  important due to increased data sparsity in certain regions. 
In addition, we intend to improve the efficiency of our calibration procedure by eliminating the need to split the calibration set. To this end, we will explore adaptations based on Jackknife+ and CV+ methods \citep{barber2021predictive}, as well as multiple-split conformal prediction techniques \citep{Lei2018}, which offer ways to aggregate information with statistical guarantees. 
Code to implement \ourmethod \ and reproduce the experiments is available at \url{https://github.com/Monoxido45/EPICSCORE}.

\begin{acknowledgements} 

This study was financed in part by the Coordenação de Aperfeiçoamento de Pessoal de Nível Superior - Brasil (CAPES) - Finance Code 001.
    L.M.C.C is grateful for the fellowship provided by São Paulo Research Foundation (FAPESP), grant 2022/08579-7. V. S. S. is grateful for the financial support of FAPESP (grant 2023/05587-1). R. I. is grateful for the financial support of FAPESP (grants 2019/11321-9 and 2023/07068-1) and 
CNPq (grants 422705/2021-7 and 305065/2023-8). The authors are also grateful to Rodrigo F. L. Lassance for his suggestions and insightful discussions. 

\end{acknowledgements}

\bibliography{uai2025-template}

\newpage

\onecolumn

\title{Epistemic Uncertainty in Conformal Scores: A Unified Approach\\(Appendix)}
\maketitle

\appendix

\section{Technical Details and Supplementary Results for the Introduction's Examples}  
\label{sec::technical}

\subsection{Regression}
\label{sec::technicalReg}

In this section, we detail the example presented by Figure \ref{fig::reg_split}. We simulate a scenario with two distinct dense regions exhibiting low aleatoric and epistemic uncertainty, separated by an intermediate, sparser region with high aleatoric and epistemic uncertainty. Given a sample size $n$, we first generate $\left \lfloor{0.425 \cdot n}\right \rfloor$ samples for each $X \sim U(0, 1.5)$ and $X \sim U(8, 10)$, ensuring that $85\%$ of the data comes from the two outer regions, reflecting low epistemic uncertainty. The corresponding response variable follows $Y \sim N(2\sin{X}, 0.1)$, which also reflects low aleatoric uncertainty. For the remaining $\left \lfloor{0.15 \cdot n}\right \rfloor$ samples, we draw $X$ from a transformed Beta distribution, $X \sim (\text{Beta}(8, 8) \cdot (8 - 1.5) ) + 1.5$, concentrating points in the intermediate region. Here, the response variable follows $Y \sim N(2 \sin{X}, 2.1)$, introducing high aleatoric uncertainty. Epistemic uncertainty is particularly elevated at the boundaries of this region.

In this setting, we use a K-nearest neighbors (KNN) algorithm with $k = 10$ as the regression base model $g(\x)$. We illustrate the difference between \ourmethod{} and established conformal prediction baselines for regression intervals, including Regression Split \citep{lei2014distribution}, Weighted Regression Split \citep{Lei2018}, and the Mondrian Conformal Regression \citep{bostrom2020mondrian}. For \ourmethod{}, we use the BART-based version, with $m = 100$ trees, default prior options for all parameters, and a heteroscedastic gamma distribution as the probability model for the conformal score (detailed in \ref{sec::bart_details}), which is appropriate given that the regression conformal score is non-negative and often asymmetric.

In terms of baselines, both the Weighted and Mondrian methods estimate the conditional spread to construct prediction intervals. The locally weighted approach models the Mean Absolute Deviance (MAD), $\E[|g(\X) - Y||\X]$, by regressing absolute residuals $|g(\x) - y|$ on $\X$ using the same model type as $g(\x)$. Meanwhile, the Mondrian method partitions the feature space using a binning scheme (or taxonomy) based on an estimation of conditional variance $\V[Y|\X]$, generally obtained using ensemble-based variance, commonly derived from an additionally fitted Random Forest. 

Visually, both baseline methods outperform regression split, providing adaptive prediction intervals that widen in regions with high aleatoric uncertainty and narrow in regions with low aleatoric uncertainty. However, both methods struggle to generate wider intervals in data-sparse regions, such as $x \in (1.5, 2)$ and $x \in (7,8)$. This limitation arises because these regions have low spread estimates due to insufficient data, leading to underestimated cutoffs. In contrast, \ourmethod{} offers widened predictive intervals in these regions, better capturing epistemic uncertainty, while still accurately representing uncertainty in data-rich areas.

\subsection{Classification}
\label{sec::technicalClass}
For the image classification example (introduced in Figure \ref{fig::images}), we used the publicly available CIFAR-100 dataset \citep{krizhevsky2009learning}, which consists of $60,000$ color images of size $32\times32$ spanning $100$ classes, with each class containing $600$ images. The dataset was split into training, calibration, and test sets, allocating $10\%$ for testing and $45\%$ each for training and calibration. We utilized a ResNet-34 model \citep{he2016deep} to extract 512-dimensional feature representations and trained a Random Forest classifier on the training set with default parameters except for the number of trees, which we fixed at $300$. The classifier reports an accuracy of $47\%$ in the test set. The Adaptive Prediction Set and \ourmethod{} were then applied to the calibration set, using the MDN MC-dropout variant of \ourmethod{}, which model's architecture is better detailed in \ref{sec::mc-dropout-architecture}. 

In the classification setting, \ourmethod{} includes an alternative adaptation to handle the discrete nature of the conformal score (see Section \ref{sec::special_cases} for details). However, it can also be applied similarly to the regression and quantile regression settings by treating the conformal score as continuous, fitting a predictive distribution, and deriving adaptive cutoffs. This approach involves normalizing the score to better leverage the chosen models. Given the large number of classes in CIFAR-100, this approximation remains valid, as it produces more fine-grained probability vectors. In this example, we adopt this formulation of \ourmethod{}, using the APS score as the conformal score and deriving adaptive thresholds from its fitted predictive distribution. 

To provide both a broad and detailed performance comparison, we first assess each method’s coverage and set size using the size-stratified coverage (SSC) metric \citep{angelopoulos2021gentle}. Next, we examine and visualize the prediction set sizes for the top 150 outliers and inliers, evaluating how well each method captures epistemic uncertainty in outliers while remaining adaptive for inliers. For illustration purposes, we set $\alpha = 0.2$ across all analyses.

\subsubsection*{SSC metric}
The SSC metric aims to evaluate the calibration of prediction sets by stratifying them into $G$ bins $\{B_j\}_{j = 1}^G$ based on their cardinality. For $j < G$, the bin $B_j$ contains the prediction sets with cardinality $j$,  while $B_G$ includes all sets with at least $G$ elements. Formally, let  $I_j = \{i: R(\X_i) \in B_j\}$ denote the indices of prediction sets that fall into bin $B_j$. The SSC metric for a given prediction set method $R(\cdot)$ is then defined as:
\begin{align}
    SSC(R) = \min_{j \in \{1, \dots, G\}} \frac{1}{|I_j|} \sum_{i \in I_j} \I \left\{Y_i \in R(\X_i) \right\} \; .
\end{align}
Intuitively, this metric measures the minimum coverage of $R(\cdot)$ across different set sizes, assessing whether coverage remains stable despite changes in set cardinality. SSC values close to $1 - \alpha$ indicate strong coverage performance, while values farther from $1 - \alpha$ suggest greater violations of conditional coverage \citep{angelopoulos2021gentle}. For this analysis, we set $G = 15$. Table \ref{tab:ssc_methods} reports the SSC average values and 2 times the standard error across 10 runs for both methods.



\begin{table}[!h]
\centering
\caption{Average SSC metric over 10 runs, with twice the standard deviation in brackets. \ourmethod{} achieves SSC values closer to the nominal level compared to APS.}
\label{tab:ssc_methods}
\begin{tabular}{lrr}
\toprule
\textbf{Method} & \textbf{SSC} & \textbf{2 * SE} \\
\midrule
\ourmethod-MDN & 0.734 & 0.014 \\
APS & 0.553  & 0.036 \\
\bottomrule
\end{tabular}
\end{table}


These results show that \ourmethod-MDN achieves an SSC much closer to the nominal level of $0.8$ than APS, indicating more consistent coverage across different set sizes. In contrast, APS has a lower average SSC, reflecting greater deviations from target coverage and potentially less reliable uncertainty quantification.

\subsubsection*{Outlier and inlier analysis}
To differentiate inliers from outliers for further analysis and comparison of each method's prediction sets, we first apply t-SNE \citep{van2008visualizing} to reduce the dimensionality of the feature space in the test set. We then use the Local Outlier Factor (LOF) method \citep{breunig2000lof} for outlier detection, leveraging a KNN-based density estimation to identify anomalies. The LOF score not only highlights outliers but also helps characterize typical (inlier) observations, offering a structured approach to assessing epistemic uncertainty across different regions of the data distribution. We fit LOF on the first two t-SNE dimensions, assuming a contamination rate of $10\%$ (i.e., treating 
$10\%$ of the sample as outliers), and rank the top 150 outliers and inliers for analysis.
 
In general, we expect outliers to have wider prediction sets, as they are located in sparser regions of the feature space. In contrast, inliers are likely to have narrower sets, reflecting their position in denser, more typical regions of the data distribution. Beyond Figure \ref{fig::images}, Figures \ref{fig::outlier_images_prediction_sets} and \ref{fig::inlier_images_prediction_sets} provide additional examples that further illustrate this behavior, emphasizing how \ourmethod{} differentiates itself from APS. Additionally, Figure \ref{fig::outliers_inliers_set_sizes_boxplot} displays the distribution of set sizes for outliers and inliers across both methods. While both methods generate larger prediction sets for outliers than for inliers, \ourmethod{} shows higher set sizes and a more dispersed distribution for outliers compared to APS, while presenting a more concentrated distribution for inliers.

Furthermore, we observe that the APS set sizes do not exceed a cardinality of 20, highlighting its lack of adaptability. Overall, these results reinforce the flexibility of our method across regions with varying levels of epistemic uncertainty, showcasing its advantage over the standard APS by explicitly incorporating epistemic uncertainty into the cutoff derivation.

\begin{figure}[ht]
    \centering
    \includegraphics[width=0.8\linewidth]{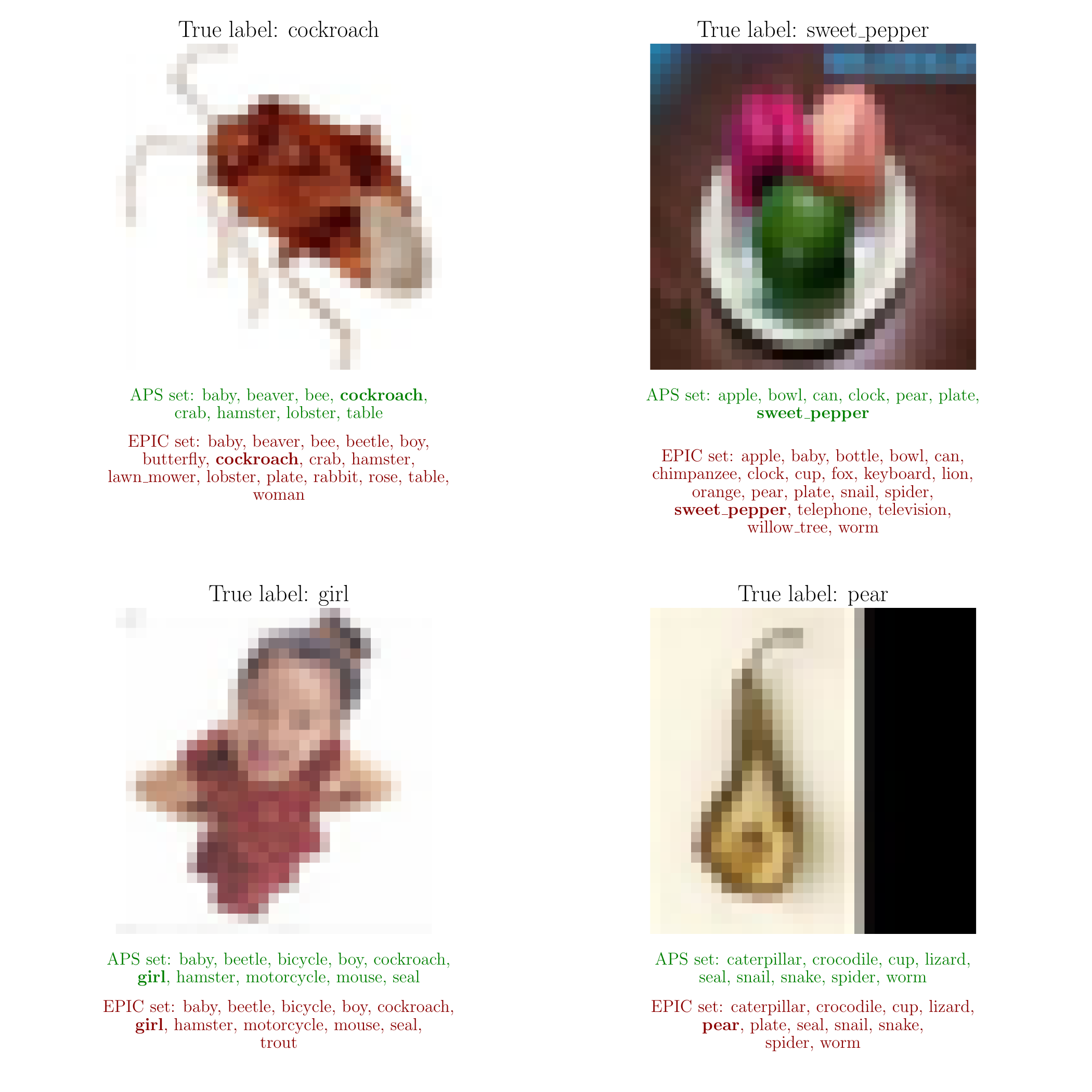}
    \caption{Additional outlier image prediction sets examples. \ourmethod{} consistently produces broader prediction sets for all selected outlier images, effectively capturing the high epistemic uncertainty associated with these observations.}
    \label{fig::outlier_images_prediction_sets}
\end{figure}

\begin{figure}[ht]
    \centering
    \includegraphics[width=0.8\linewidth]{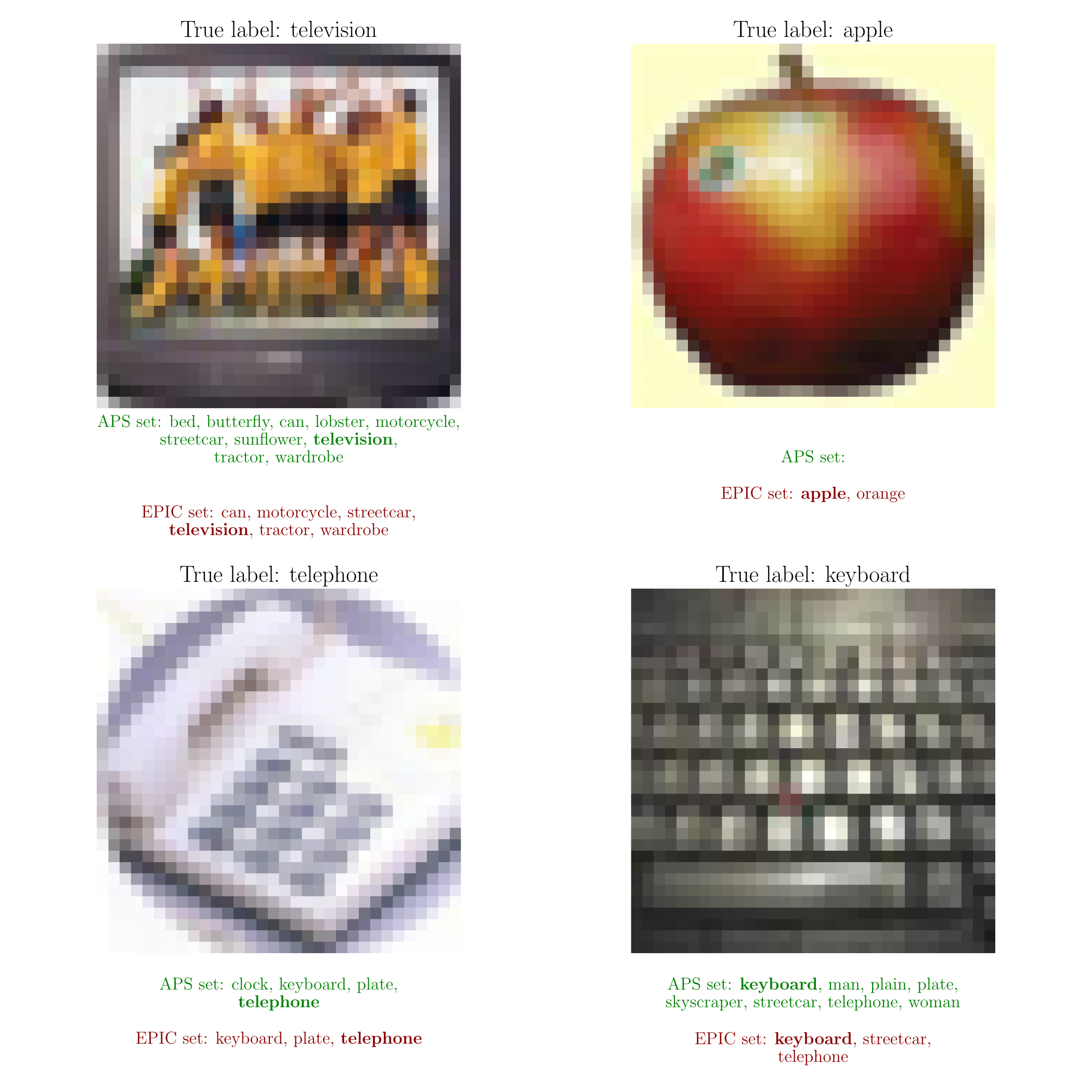}
    \caption{Additional inlier prediction sets examples. \ourmethod{} generates more compact prediction sets for all selected inliers while also preventing empty sets in one instance. This highlights its robustness and reliability in regions with low epistemic uncertainty.}
    \label{fig::inlier_images_prediction_sets}
\end{figure}

\begin{figure}[ht]
    \centering
    \includegraphics[width=0.675\linewidth]{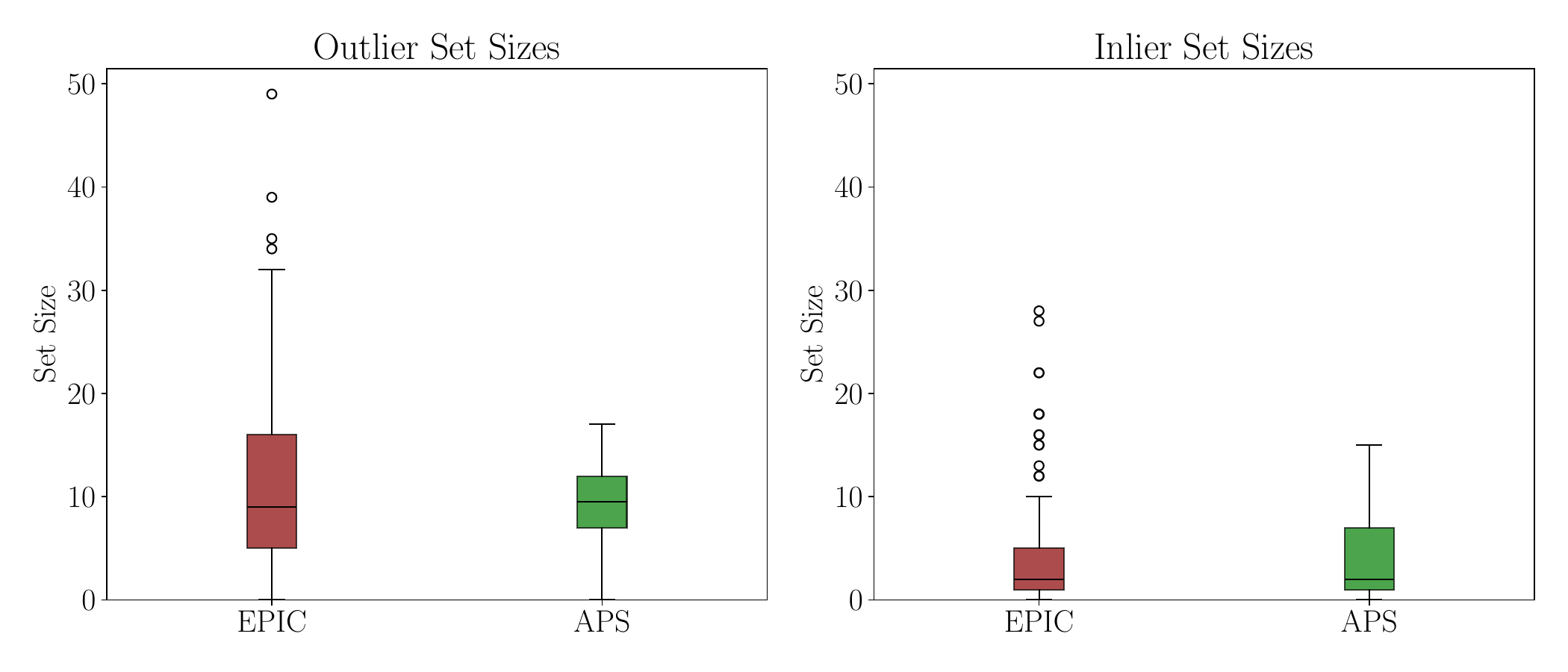}
    \caption{Left: Prediction set sizes for the top 150 most outlying observations. Right: Prediction set sizes for the top 150 most typical (inlier) observations. Both methods consistently produce larger prediction sets for outliers compared to inliers, but \ourmethod{} shows a more dispersed boxplot with higher set sizes for outliers, and a more concentrated boxplot for inliers compared to APS.}
    \label{fig::outliers_inliers_set_sizes_boxplot}
\end{figure}

\newpage

\section{Predictive Distributions via Monte Carlo Dropout and Batch Normalization}
\label{sec::mc_dropout_details}

Monte Carlo (MC) dropout was originally introduced as a regularization technique to prevent neural network overfitting \citep{srivastava2014dropout}. However, it can also be used to approximate the predictive distribution by interpreting dropout as a variational Bayesian method. Indeed, \cite{gal2016dropout} demonstrated that applying dropout at both training and prediction corresponds to a variational approximation to the posterior over the network weights. Specifically, the dropout masks define a variational family where each mask corresponds to a Bernoulli-distributed perturbation of the network.
From a practical standpoint, the predictive distribution is obtained by performing multiple stochastic forward passes with dropout enabled at prediction time \citep{Izbicki2025}. Each pass applies a different dropout mask, resulting in varied outputs—in our case, different sets of means, variances, and weights of a Gaussian mixture for the same data point $\x$. A sample of $Y$ is then drawn from each configuration, producing a Monte Carlo sample from the predictive distribution $Y|\x,D$.

Batch normalization (BN)  \citep{ioffebatch2015} also introduces stochasticity, but through its reliance on mini-batch statistics during training. As shown by \cite{teye2018bayesian}, this stochasticity induces an implicit distribution over the network's parameters, which can be interpreted as a form of approximate Bayesian prediction. During prediction, the batch statistics are fixed, and the model's outputs can be treated as samples from an approximate posterior—similarly to the MC dropout case.


\section{Additional Results and Details for Real Data Experiments}
\label{sec::additional_results}
In this section, we provide an overview of the evaluation metrics, dataset summaries, and additional results for both quantile and standard regression experiments. Additionally, we outline the architecture and hyperparameter configurations for each base model.

\subsection{Evaluation Metrics}
\label{sec::evaluation_metrics}

Let $\hat{R}(\cdot)$ denote a generic prediction interval. Given a test set $(\mathbf{X}_1, Y_1), (\mathbf{X}_2, Y_2), \ldots, (\mathbf{X}_m, Y_m)$, we evaluate performance using the following metrics:

\begin{itemize}
    \item \textbf{Average Marginal Coverage (AMC)}:
    \begin{equation*}
        \text{AMC} = \frac{1}{m} \sum_{i=1}^{m} \I\left( Y_i \in \hat{R}(\X_i) \right),
    \end{equation*}
    which is an estimate of the marginal coverage of $R$.
    
    \item \textbf{Average Interval Score Loss (AISL)} \citep{gneiting2007strictly}: 
    \begin{align*}
        \text{AISL} = \frac{1}{m} \sum_{i=1}^{m} \Bigg[ &\left( \max \hat{R}(\X_i) - \min \hat{R}(\X_i)\right) \\
        &+ \frac{2}{\alpha} \cdot \left(\min \hat{R}(\X_i) - Y_i \right) \cdot \I \left\{Y_i < \min \hat{R}(\X_i) \right\} \\
        &+ \frac{2}{\alpha} \cdot \left(Y_i - \max \hat{R}(\X_i) \right)\cdot \I \left\{Y_i > \max \hat{R}(\X_i \right\} \Bigg] \; ,
    \end{align*}
    where $\min \hat{R}(\X)$ and  $\max \hat{R}(\X)$  represent the lower and upper bounds of the prediction interval, respectively, and $\alpha$ is the miscalibration level. The Interval Score Loss balances two key objectives: maintaining narrow prediction intervals while penalizing those that fail to cover $Y_i$,  with larger penalties for greater deviations. By averaging these scores across all instances, AISL provides a measure that prioritizes the shortest interval while ensuring sufficient coverage. We chose the AISL following prior work in the literature, where it is commonly used as a summary metric \citep{rossellini2024integrating}, a paper that also addresses epistemic uncertainty in conformal predictions.

    \item \textbf{Interval Length (IL)}:
    \begin{equation*}
        \text{IL} =  \frac{1}{m} \sum_{i = 1}^{m} \max \hat{R}(\X_i) - \min \hat{R}(\X_i) \; ,
    \end{equation*}
    which measures the average interval length, reflecting the precision of the predictive intervals. Larger values correspond to wider, less informative intervals, while smaller values indicate more compact and precise intervals.
\item \textbf{Pearson Correlation between Coverage and Interval Length ($\rho$)} \citep{feldman2021improving}: This metric measures the correlation between the width of the prediction interval and its coverage, providing insight into potential conditional coverage violations. Specifically,

\begin{equation*} \rho = \left| \frac{\text{Cov}(\mathbf{C}, \mathbf{W})}{\sigma_{\mathbf{C}} \sigma_{\mathbf{W}}} \right| , \end{equation*}

where $\mathbf{C} = (C_1, \ldots, C_m)$ represents a binary vector, with $C_i = \I(Y_i \in \hat{R}(\mathbf{X}_i))$ indicating whether the prediction interval $\hat{R}(\mathbf{X}_i)$ covers $Y_i$, and $\mathbf{W} = (W_1, \ldots, W_m)$, where $W_i = \max \hat{R}(\mathbf{X}_i) - \min \hat{R}(\mathbf{X}_i)$. 
A strong correlation between coverage and interval width suggests a potential violation of conditional coverage, which requires their independence \citep{feldman2021improving}. However, $\rho = 0$ does not guarantee conditional coverage, as non-adaptive methods like regression split can achieve zero correlation by maintaining constant-width intervals \citep{rossellini2024integrating}. Thus, while this metric provides a useful proxy for assessing conditional coverage, it is not a definitive measure.

\item \textbf{Outlier-to-inlier interval length ratio (ILR)}:
\begin{align*}
    \text{ILR} = \frac{\frac{1}{|I_{\text{out}}|} \sum_{i \in I_{\text{out}}} \max \hat{R}(\X_i) - \min \hat{R}(\X_i)}{\frac{1}{|I_{\text{in}}|} \sum_{i \in I_{\text{in}}} \max \hat{R}(\X_i) - \min \hat{R}(\X_i)} \; ,
\end{align*}
where $I_{\text{out}}$ and $I_{\text{in}}$ denote the sets of outlier and inlier indices in the test set, respectively.Following the procedure described in Section~\ref{sec::technicalClass}, we distinguish inliers from outliers by first reducing dimensionality with t-SNE \citep{van2008visualizing}, and then applying the Local Outlier Factor (LOF) method \citep{breunig2000lof} using a contamination rate of $5\%$ (i.e., treating $5\%$ of the test instances as outliers). For this metric, we include all detected outliers in $I_{\text{out}}$ and define inliers $I_{\text{in}}$ as the 15\% of instances with the lowest LOF scores. This metric captures how adaptively a method responds to epistemic uncertainty: higher ILR values indicate that outliers—typically found in data-sparse regions—are assigned wider prediction intervals relative to inliers, as expected from a well-calibrated uncertainty-aware method.

\item \textbf{Average coverage on outliers (ACO)}:
\begin{align*}
    \text{ACO} = \frac{1}{|I_{\text{out}}|} \sum_{i \in I_{\text{out}}} \I(Y_i \in \hat{R}(\X_i)) \; ,
\end{align*}
where $I_{\text{out}}$ is the set of outlier indices in the test set, obtained using the same outlier detection procedure described above. This metric measures the proportion of outlier instances whose true response $Y_i$ falls within the predtion interval $\hat{R}(\cdot)$. Higher ACO values—ideally close to the nominal level $1 -\alpha$ indicate that the method maintains reliable coverage even in data-sparse regions. As such, ACO complements the ILR metric by evaluating whether the adaptively wider intervals assigned to outliers are indeed well-calibrated.
\end{itemize}

\subsection{Additional results}
\label{sec:additional}
Table~\ref{tab:realdata} summarizes the dataset details. Quantile regression results are reported in Tables~\ref{tab:amc}, \ref{tab:il}, \ref{tab:pcorr}, and \ref{tab:coverage_outlier}, while standard regression results appear in Tables~\ref{tab:amc_reg}, \ref{tab:il_reg}, and \ref{tab:pcorr_reg}. Additional metrics on outlier-to-inlier interval length ratios and outlier coverage are presented in Tables~\ref{tab:coverage_outlier} (quantile regression), \ref{tab:interval_ratio_reg}, and \ref{tab:coverage_outlier_reg} (regression).

\begin{table*}[ht]
\caption{Summary of the datasets used in this paper, including the number of samples ($n$), features ($p$), and access links.}
\label{tab:realdata}
\centering
\begin{adjustbox}{max width=0.85\textwidth}
\begin{tabular}{lccccccccccccp{15mm}}
\hline
\textbf{Dataset} & n & p & Source& \textbf{Dataset} & n & p &Source  \\
\hline
Airfoil & 1503  & 5  & \href{https://archive.ics.uci.edu/dataset/291/airfoil+self+noise}{Airfoil (UCI)} & Protein  & 45730 & 8 & \href{http://archive.ics.uci.edu/dataset/265/physicochemical+properties+of+protein+tertiary+structure}{Protein (UCI)}   \\

Bike & 10885 & 12 & \href{https://www.kaggle.com/code/rajmehra03/bike-sharing-demand-rmsle-0-3194/input?select=train.csv}{Bike (Kaggle)} & Star & 2161  & 48 & \href{https://dataverse.harvard.edu/dataset.xhtml?persistentId=doi:10.7910/DVN/SIWH9F}{Star (Harvard Dataverse)} \\

Concrete & 1030  & 8 & \href{https://archive.ics.uci.edu/dataset/165/concrete+compressive+strength}{Concrete (UCI)} & SuperConductivity & 21263 & 81 & \href{http://archive.ics.uci.edu/ml/datasets/Superconductivty+Data}{Superconductivity (UCI)} \\

Cycle & 9568 & 4 & \href{http://archive.ics.uci.edu/dataset/294/combined+cycle+power+plant}{Cycle (UCI)}  & Wave Energy Converter & 54000  & 49 & {\href{https://archive.ics.uci.edu/dataset/882/large-scale+wave+energy+farm}{WEC (UCI)}}  \\

Homes & 21613 & 17  & \href{https://www.kaggle.com/datasets/harlfoxem/housesalesprediction}{Home (Kaggle)} & Winered & 4898  & 11 & \href{https://archive.ics.uci.edu/dataset/186/wine+quality}{Wine red (UCI)} \\

Eletric & 10000 & 12  & \href{http://archive.ics.uci.edu/ml/datasets/Electrical+Grid+Stability+Simulated+Data+}{Electric (UCI)}  & WineWhite& 1599  & 11 & \href{https://archive.ics.uci.edu/dataset/186/wine+quality}{Wine white (UCI)}  \\

Meps19& 15781  & 141 & Meps19 \href{https://meps.ahrq.gov/mepsweb/data_stats/download_data_files_detail.jsp?cboPufNumber=HC-181}{(AHRQ site)}) & &   & &  \\
\hline
\end{tabular}
\end{adjustbox}
\end{table*}

\label{sec::additional_res}
\begin{table*}[h]
\caption{Quantile regression Mean Average Coverage values across different methods and datasets. The reported values represent the average over 50 runs, with two times the standard deviation in parentheses. As expected for conformal methods, all approaches achieve marginal coverage close to the nominal level of 0.9}
\label{tab:amc}
\centering
\begin{adjustbox}{max width=\textwidth}
\begin{tabular}{lccccccc}
\hline
\textbf{Dataset}  & \textbf{EPIC-BART} & \textbf{EPIC-GP} & \textbf{EPIC-MDN} & \textbf{CQR}  & \textbf{CQR-r} & \textbf{UACQR-P} & \textbf{UACQR-S} \\ \hline
airfoil           & 0.9 (0.008)        & 0.9 (0.01)       & 0.896 (0.01)      & 0.901 (0.007) & 0.901 (0.007)  & 0.907 (0.009)    & 0.9 (0.007)      \\
bike              & 0.9 (0.003)        & 0.898 (0.003)    & 0.899 (0.003)     & 0.899 (0.002) & 0.899 (0.002)  & 0.9 (0.002)      & 0.9 (0.002)      \\
concrete          & 0.905 (0.008)      & 0.9 (0.009)      & 0.898 (0.01)      & 0.897 (0.007) & 0.897 (0.007)  & 0.914 (0.012)    & 0.895 (0.007)    \\
cycle             & 0.899 (0.003)      & 0.9 (0.003)      & 0.898 (0.004)     & 0.901 (0.003) & 0.902 (0.003)  & 0.901 (0.002)    & 0.9 (0.002)      \\
electric          & 0.9 (0.003)        & 0.901 (0.003)    & 0.902 (0.004)     & 0.901 (0.002) & 0.901 (0.002)  & 0.901 (0.002)    & 0.901 (0.002)    \\
homes             & 0.902 (0.003)      & 0.902 (0.003)    & 0.9 (0.003)       & 0.901 (0.002) & 0.901 (0.002)  & 0.901 (0.002)    & 0.901 (0.002)    \\
meps19            & 0.9 (0.003)        & 0.9 (0.003)      & 0.9 (0.003)       & 0.899 (0.002) & 0.899 (0.002)  & 0.901 (0.002)    & 0.899 (0.002)    \\
protein           & 0.897 (0.003)      & 0.897 (0.003)    & 0.897 (0.003)     & 0.9 (0.001)   & 0.9 (0.001)    & 0.901 (0.001)    & 0.9 (0.001)      \\
star              & 0.902 (0.006)      & 0.902 (0.006)    & 0.903 (0.006)     & 0.902 (0.004) & 0.901 (0.004)  & 0.93 (0.013)     & 0.902 (0.004)    \\
superconductivity & 0.898 (0.004)      & 0.898 (0.003)    & 0.898 (0.003)     & 0.9 (0.002)   & 0.9 (0.002)    & 0.9 (0.001)      & 0.9 (0.002)      \\
WEC               & 0.897 (0.003)      & 0.899 (0.003)    & 0.897 (0.004)     & 0.9 (0.001)   & 0.9 (0.001)    & 0.899 (0.001)    & 0.9 (0.001)      
\\

winered           & 0.906 (0.008)      & 0.905 (0.008)    & 0.904 (0.007)     & 0.897 (0.006) & 0.897 (0.006)  & 0.903 (0.009)    & 0.897 (0.006)    \\
winewhite         & 0.901 (0.004)      & 0.9 (0.005)      & 0.9 (0.004)       & 0.898 (0.003) & 0.898 (0.003)  & 0.908 (0.009)    & 0.898 (0.003)    \\ \hline
\end{tabular}
\end{adjustbox}
\end{table*}

\begin{table}[h]
\caption{Quantile regression Interval Length values across different methods and datasets. The reported values represent the average over 50 runs, with two times the standard deviation in parentheses. Bolded values indicate the best-performing method within a 95\% confidence interval. Overall, \ourmethod{} consistently produces narrower intervals in most cases.}
\label{tab:il}
\centering
\begin{adjustbox}{max width=\textwidth}
\begin{tabular}{lccccccc}
\hline
\textbf{Dataset}  & \textbf{EPIC-BART}      & \textbf{EPIC-GP}        & \textbf{EPIC-MDN}       & \textbf{CQR}            & \textbf{CQR-r}           & \textbf{UACQR-P}        & \textbf{UACQR-S}        \\ \hline
airfoil           & 16.521 (0.18)           & \textbf{16.395 (0.237)} & \textbf{16.02 (0.222)}  & 17.087 (0.127)          & 17.09 (0.124)            & 18.838 (0.367)          & 16.656 (0.386)          \\
bike $\times (10^1)$             & 37.042 (0.190)          & 38.413 (0.229)          & \textbf{36.250 (0.218)} & 41.164 (0.150)          & 41.125 (0.153)           & 43.627 (0.591)          & 37.415 (0.386)          \\
concrete          & \textbf{36.537 (0.375)} & 38.328 (0.618)          & 37.614 (0.651)          & 39.477 (0.353)          & 39.486 (0.36)            & 44.425 (1.307)          & 39.853 (1.536)          \\
cycle             & \textbf{30.975 (0.128)} & 31.587 (0.132)          & \textbf{30.714 (0.146)} & 35.235 (0.095)          & 35.346 (0.093)           & 38.292 (0.195)          & 31.045 (0.207)          \\
electric          & 0.088 (0.001)           & 0.084 (0.001)           & \textbf{0.072 (0.001)}  & 0.09 (0.001)            & 0.09 (0.001)             & 0.097 (0.001)           & 0.084 (0.001)           \\
homes $\times (10^5)$ & 5.888 (0.028)           & 5.739 (0.028)           & 5.816 (0.040)           & 6.313 (0.024)           & 6.259 (0.024)            & 6.750 (0.0302)          & \textbf{5.312 (0.0309)} \\
meps19            & 32.996 (0.314)          & 29.268 (0.262)          & 29.16 (0.266)           & 28.948 (0.249)          & 28.949 (0.249)           & \textbf{27.857 (0.314)} & 32.763 (0.815)          \\
protein           & 16.195 (0.039)          & 16.485 (0.052)          & \textbf{16.048 (0.034)} & 16.378 (0.011)          & 16.378 (0.011)           & 16.797 (0.019)          & 16.356 (0.017)          \\
star $\times (10^1)$             & \textbf{81.851 (1.029)} & \textbf{81.760 (1.042)} & \textbf{82.050 (1.083)} & \textbf{81.359 (0.508)} & \textbf{81.396 (0.5117)} & 82.521 (0.618)          & 83.253 (0.952)          \\
superconductivity & 66.906 (0.205)          & 70.631 (0.366)          & \textbf{64.805 (0.197)} & 69.51 (0.144)           & 69.482 (0.145)           & 78.478 (0.7)            & 67.046 (0.492)          \\
WEC $\times (10^5)$              & 2.401 (0.0111)          & 2.076 (0.011)          & \textbf{1.890 (0.009)}  & 2.708 (0.004)           & 2.709 (0.004)            & 2.843 (0.008)           & 2.547 (0.007)           \\

winered           & 2.098 (0.034)           & 2.096 (0.035)           & 2.11 (0.031)            & \textbf{1.906 (0.011)}  & \textbf{1.902 (0.01)}    & 2.031 (0.025)           & 2.077 (0.042)           \\
winewhite         & 2.31 (0.017)            & 2.212 (0.029)           & 2.253 (0.016)           & \textbf{2.12 (0.006)}   & \textbf{2.121 (0.006)}   & \textbf{2.124 (0.012)}  & 2.222 (0.017)           \\ \hline
\end{tabular}
\end{adjustbox}
\end{table}

\begin{table}[h]
\caption{Quantile regression Pearson correlation values across different methods and datasets. The reported values represent the average over 50 runs, with two times the standard deviation in parentheses. Bolded values indicate the best-performing method within a $95\%$ confidence interval. Overall, \ourmethod{} exhibits low correlation in most cases, reflecting strong conditional coverage performance.}
\label{tab:pcorr}
\centering
\begin{adjustbox}{max width = 0.925\textwidth}
\begin{tabular}{lccccccc}
\hline
\textbf{Dataset}  & \textbf{EPIC-BART}     & \textbf{EPIC-GP}       & \textbf{EPIC-MDN}      & \textbf{CQR}           & \textbf{CQR-r}         & \textbf{UACQR-P}       & \textbf{UACQR-S}       \\ \hline
airfoil           & \textbf{0.06 (0.013)}  & 0.18 (0.016)           & \textbf{0.071 (0.016)} & 0.125 (0.017)          & 0.129 (0.017)          & 0.132 (0.033)          & 0.108 (0.02)           \\
bike              & 0.171 (0.013)          & 0.138 (0.013)          & 0.213 (0.012)          & \textbf{0.062 (0.013)} & \textbf{0.069 (0.013)} & 0.108 (0.016)          & 0.091 (0.012)          \\
concrete          & 0.101 (0.021)          & 0.147 (0.02)           & \textbf{0.068 (0.013)} & \textbf{0.081 (0.017)} & \textbf{0.082 (0.017)} & 0.121 (0.034)          & \textbf{0.088 (0.02)}  \\
cycle             & \textbf{0.046 (0.008)} & \textbf{0.045 (0.009)} & 0.085 (0.01)           & 0.27 (0.011)           & 0.292 (0.011)          & 0.255 (0.011)          & 0.192 (0.012)          \\
electric          & \textbf{0.023 (0.006)} & \textbf{0.044 (0.009)} & 0.159 (0.008)          & 0.075 (0.006)          & 0.071 (0.006)          & 0.123 (0.01)           & 0.134 (0.008)          \\
homes             & 0.122 (0.007)          & 0.143 (0.01)           & \textbf{0.048 (0.008)} & 0.126 (0.008)          & 0.15 (0.007)           & 0.271 (0.008)          & 0.221 (0.007)          \\
meps19            & \textbf{0.017 (0.004)} & 0.084 (0.01)           & 0.069 (0.009)          & 0.08 (0.006)           & 0.08 (0.006)           & 0.128 (0.006)          & 0.051 (0.006)          \\
protein           & \textbf{0.031 (0.004)} & 0.15 (0.01)            & 0.073 (0.004)          & 0.094 (0.003)          & 0.094 (0.003)          & 0.116 (0.004)          & 0.094 (0.004)          \\
star              & 0.07 (0.012)           & \textbf{0.042 (0.009)} & 0.085 (0.012)          & \textbf{0.046 (0.009)} & \textbf{0.047 (0.01)}  & \textbf{0.048 (0.011)} & \textbf{0.041 (0.009)} \\
superconductivity & 0.167 (0.006)          & 0.217 (0.008)          & \textbf{0.034 (0.007)} & 0.069 (0.007)          & 0.073 (0.006)          & 0.137 (0.008)          & 0.088 (0.005)          \\
WEC               & \textbf{0.122 (0.004)} & \textbf{0.119 (0.006)} & 0.149 (0.005)          & 0.136 (0.005)          & 0.147 (0.005)          & \textbf{0.13 (0.007)}  & 0.132 (0.005)          \\
winered           & \textbf{0.062 (0.013)} & \textbf{0.085 (0.016)} & \textbf{0.058 (0.011)} & 0.113 (0.015)          & 0.114 (0.016)          & 0.097 (0.019)          & \textbf{0.076 (0.016)} \\
winewhite         & \textbf{0.068 (0.016)} & 0.13 (0.016)           & \textbf{0.072 (0.012)} & 0.147 (0.011)          & 0.147 (0.011)          & 0.156 (0.014)          & 0.099 (0.011)          \\ \hline
\end{tabular}
\end{adjustbox}
\end{table}

\begin{table}[h]
    \caption{Average coverage on outliers across methods and datasets in the quantile regression setting. This table reports the average prediction interval coverage for outlier observations—identified as data-sparse points—across 50 runs. Bold values indicate the best-performing method within a 95\% confidence interval. \ourmethod{} consistently delivers near-nominal coverage across most datasets.}
\label{tab:coverage_outlier}
\centering
\begin{adjustbox}{max width = 0.925\textwidth}
\begin{tabular}{lccccccc}
\hline
\textbf{Dataset}  & \textbf{EPIC-BART}     & \textbf{EPIC-GP}       & \textbf{EPIC-MDN}      & \textbf{CQR}           & \textbf{CQR-r}         & \textbf{UACQR-P}       & \textbf{UACQR-S}       \\ \hline
airfoil           & \textbf{0.894 (0.029)} & \textbf{0.875 (0.032)} & \textbf{0.902 (0.024)} & \textbf{0.877 (0.031)} & \textbf{0.876 (0.032)} & \textbf{0.875 (0.033)} & \textbf{0.872 (0.031)} \\
bike              & \textbf{0.885 (0.009)} & \textbf{0.9 (0.011)}   & \textbf{0.894 (0.011)} & \textbf{0.882 (0.01)}  & \textbf{0.882 (0.01)}  & \textbf{0.887 (0.009)} & \textbf{0.885 (0.009)} \\
concrete          & \textbf{0.824 (0.039)} & \textbf{0.855 (0.034)} & \textbf{0.875 (0.038)} & \textbf{0.845 (0.035)} & \textbf{0.844 (0.036)} & \textbf{0.878 (0.032)} & \textbf{0.845 (0.034)} \\
cycle             & \textbf{0.907 (0.009)} & \textbf{0.908 (0.01)}  & \textbf{0.906 (0.011)} & \textbf{0.9 (0.009)}   & \textbf{0.901 (0.009)} & \textbf{0.909 (0.01)}  & \textbf{0.906 (0.009)} \\
electric          & \textbf{0.887 (0.01)}  & \textbf{0.887 (0.009)} & \textbf{0.895 (0.009)} & \textbf{0.888 (0.009)} & \textbf{0.888 (0.009)} & \textbf{0.889 (0.01)}  & \textbf{0.889 (0.009)} \\
homes             & \textbf{0.89 (0.007)}  & \textbf{0.885 (0.007)} & \textbf{0.895 (0.008)} & \textbf{0.883 (0.007)} & \textbf{0.88 (0.007)}  & 0.864 (0.009)          & 0.872 (0.007)          \\
meps19            & \textbf{0.904 (0.008)} & \textbf{0.903 (0.008)} & \textbf{0.899 (0.009)} & \textbf{0.897 (0.008)} & \textbf{0.897 (0.008)} & \textbf{0.9 (0.008)}   & \textbf{0.897 (0.008)} \\
protein           & \textbf{0.893 (0.005)} & \textbf{0.895 (0.005)} & \textbf{0.896 (0.005)} & \textbf{0.894 (0.004)} & \textbf{0.894 (0.004)} & \textbf{0.894 (0.005)} & \textbf{0.894 (0.004)} \\
star              & 0.867 (0.016)          & 0.871 (0.017)          & \textbf{0.876 (0.016)} & \textbf{0.875 (0.018)} & \textbf{0.875 (0.018)} & \textbf{0.911 (0.021)} & 0.87 (0.018)           \\
superconductivity & 0.863 (0.008)          & 0.872 (0.009)          & 0.879 (0.008)          & \textbf{0.913 (0.006)} & \textbf{0.912 (0.006)} & 0.932 (0.005)          & \textbf{0.912 (0.006)} \\
WEC               & \textbf{0.901 (0.006)} & 0.864 (0.008)          & 0.86 (0.009)           & 0.927 (0.005)          & 0.926 (0.005)          & 0.917 (0.005)          & 0.921 (0.005)          \\
winered           & \textbf{0.875 (0.029)} & \textbf{0.876 (0.03)}  & \textbf{0.892 (0.026)} & \textbf{0.846 (0.032)} & \textbf{0.847 (0.032)} & \textbf{0.861 (0.033)} & \textbf{0.848 (0.032)} \\
winewhite         & \textbf{0.86 (0.015)}  & \textbf{0.854 (0.016)} & \textbf{0.868 (0.014)} & \textbf{0.847 (0.016)} & \textbf{0.847 (0.016)} & \textbf{0.865 (0.02)}  & \textbf{0.861 (0.016)} \\ \hline
\end{tabular}
\end{adjustbox}
\end{table}

\begin{table*}[h]
\caption{Regression Mean Average Coverage values across different methods and datasets. The reported values represent the average over 50 runs, with two times the standard deviation in parentheses. As expected for conformal methods, all approaches maintain marginal coverage close to the nominal level of 0.9.}
\label{tab:amc_reg}
\centering
\begin{adjustbox}{max width = 0.85\textwidth}
\begin{tabular}{lcccccc}
\hline
\textbf{Dataset}  & \textbf{EPIC-BART} & \textbf{EPIC-GP} & \textbf{EPIC-MDN} & \textbf{Mondrian} & \textbf{Reg-split} & \textbf{Weighted} \\ \hline
airfoil           & 0.897 (0.01)      & 0.9 (0.008)     & 0.897 (0.009)    & 0.906 (0.006)     & 0.897 (0.007)      & 0.9 (0.007)       \\
bike              & 0.901 (0.003)     & 0.903 (0.003)   & 0.898 (0.003)    & 0.904 (0.002)     & 0.899 (0.002)      & 0.9 (0.002)       \\
concrete          & 0.897 (0.009)     & 0.902 (0.011)   & 0.907 (0.009)    & 0.929 (0.006)     & 0.901 (0.008)      & 0.896 (0.008)     \\
cycle             & 0.898 (0.004)     & 0.9 (0.004)     & 0.896 (0.004)    & 0.905 (0.003)     & 0.898 (0.003)      & 0.9 (0.002)       \\
electric          & 0.899 (0.003)     & 0.9 (0.003)     & 0.896 (0.003)    & 0.905 (0.002)     & 0.899 (0.003)      & 0.901 (0.003)     \\
homes             & 0.9 (0.003)       & 0.899 (0.003)   & 0.9 (0.004)      & 0.902 (0.002)     & 0.901 (0.002)      & 0.9 (0.002)       \\
meps19            & 0.899 (0.003)     & 0.9 (0.003)     & 0.897 (0.003)    & 0.902 (0.006)     & 0.9 (0.002)        & 0.9 (0.002)       \\
protein           & 0.9 (0.003)       & 0.901 (0.003)   & 0.899 (0.002)    & 0.9 (0.001)       & 0.9 (0.001)        & 0.899 (0.001)     \\
star              & 0.903 (0.005)     & 0.9 (0.006)     & 0.906 (0.006)    & 0.913 (0.005)     & 0.903 (0.004)      & 0.9 (0.004)       \\
superconductivity & 0.901 (0.003)     & 0.901 (0.003)   & 0.9 (0.004)      & 0.901 (0.002)     & 0.899 (0.002)      & 0.899 (0.002)     \\
WEC & 0.901 (0.002) & 0.898 (0.003) & 0.899 (0.003) & 0.9 (0.001) & 0.899 (0.001) & 0.9 (0.001) \\
winered           & 0.898 (0.009)     & 0.9 (0.008)     & 0.895 (0.008)    & 0.91 (0.005)      & 0.903 (0.006)      & 0.895 (0.005)     \\
winewhite         & 0.902 (0.004)     & 0.904 (0.004)   & 0.901 (0.004)    & 0.911 (0.003)     & 0.9 (0.004)        & 0.899 (0.003)     \\ \hline
\end{tabular}
\end{adjustbox}
\end{table*}

\begin{table*}[h]
\caption{Regression Interval length values across different methods and datasets. The reported values represent the average over 50 runs, with two times the standard deviation in parentheses. Bold values indicate the best-performing method within a $95\%$ confidence interval. In general, our framework produces narrower intervals in most datasets.}
\label{tab:il_reg}
\centering
\begin{adjustbox}{max width = \textwidth}
\begin{tabular}{lcccccc}
\hline
\textbf{Dataset}  & \textbf{EPIC-BART}                & \textbf{EPIC-GP}         & \textbf{EPIC-MDN}                & \textbf{Mondrian}        & \textbf{Reg-split}       & \textbf{Weighted}        \\ \hline
airfoil           & \textbf{15.099 (0.592)}           & \textbf{15.223 (0.591)}  & \textbf{15.089 (0.56)}           & 16.671 (0.632)           & \textbf{15.325 (0.469)}  & \textbf{15.693 (0.521)}  \\
bike $\times (10^1)$       & \textbf{24.910 (0.302)}           & 26.616 (0.333)           & 25.665 (0.338)                   & 27.263 (0.264)           & 27.634 (0.271)           & 25.865 (0.249)           \\
concrete          & \textbf{39.025 (1.49)}            & \textbf{39.44 (1.575)}   & \textbf{40.475 (1.551)}          & 51.284 (1.904)           & \textbf{39.943 (1.204)}  & 43.053 (1.629)           \\
cycle             & \textbf{14.911 (0.183)}           & \textbf{14.851 (0.199)}  & \textbf{14.712 (0.174)}          & \textbf{15.015 (0.164)}  & \textbf{14.855 (0.159)}  & \textbf{14.833 (0.15)}   \\
electric          & \textbf{0.036 (\textless{}0.001)} & 0.037 (\textless{}0.001) & \textbf{0.036 (\textless 0.001)} & 0.038 (\textless{}0.001) & 0.037 (\textless{}0.001) & 0.037 (\textless{}0.001) \\
homes $\times (10^5)$  & \textbf{3.848 (0.043)}            & \textbf{3.758 (0.049)}   & 4.040 (0.059)                    & 4.235 (0.037)            & 4.014 (0.033)            & 3.984 (0.032)            \\
meps19            & \textbf{25.013 (0.87)}            & \textbf{26.605 (0.774)}  & 32.093 (1.446)                   & 38.904 (1.039)           & 28.899 (0.544)           & 29.555 (0.843)           \\
protein           & 14.572 (0.106)                    & 14.311 (0.12)            & \textbf{13.573 (0.103)}          & 14.102 (0.038)           & 15.261 (0.037)           & \textbf{13.692 (0.037)}  \\
star $\times (10^1)$    & \textbf{85.3 (1.015)}         & \textbf{85.148 (1.326)} & \textbf{86.499 (1.288)}        & 90.539 (0.762)    & \textbf{85.230 (0.792)} & 104.202 (1.306)   \\
superconductivity & \textbf{39.13 (0.419)}            & \textbf{39.283 (0.479)}  & \textbf{39.547 (0.5)}            & 42.204 (0.216)           & 46.14 (0.242)            & 40.115 (0.228)           \\
WEC $\times (10^{5})$ & 0.893 (0.011) & \textbf{0.858(0.011)} & 0.900 (0.012) & 0.903 (0.004) & 0.925 (0.006) & 0.879 (0.006) \\
winered           & \textbf{2.361 (0.065)}            & \textbf{2.37 (0.067)}    & \textbf{2.316 (0.054)}           & 2.576 (0.04)             & \textbf{2.39 (0.037)}    & 2.541 (0.051)            \\
winewhite         & \textbf{2.337 (0.032)}            & \textbf{2.4 (0.032)}     & \textbf{2.356 (0.031)}           & 2.445 (0.013)            & 2.361 (0.014)            & 2.387 (0.015)            \\ \hline
\end{tabular}
\end{adjustbox}
\end{table*}

\begin{table}[h]
\caption{Regression Pearson correlation values across different methods and datasets. The reported values represent the average over 50 runs, with two times the standard deviation in parentheses. Bold values indicate the best-performing method within a $95\%$ confidence interval. The Pearson correlation for Regression Split is omitted, as its constant interval length results in an undefined correlation value. Overall, \ourmethod{} achieves low correlations in most cases, indicating strong conditional coverage performance.}
\label{tab:pcorr_reg}
\centering
\begin{adjustbox}{max width = 0.85\textwidth}
\begin{tabular}{lccccc}
\hline
\textbf{Dataset}  & \textbf{EPIC-BART}     & \textbf{EPIC-GP}      & \textbf{EPIC-MDN}      & \textbf{Mondrian}      & \textbf{Weighted}      \\ \hline
airfoil           & \textbf{0.056 (0.013)} & 0.125 (0.018)         & \textbf{0.054 (0.012)} & 0.148 (0.017)          & 0.124 (0.016)          \\
bike              & 0.164 (0.009)          & 0.172 (0.006)         & 0.054 (0.007)          & \textbf{0.028 (0.005)} & 0.043 (0.007)          \\
concrete          & \textbf{0.064 (0.015)} & 0.116 (0.019)         & \textbf{0.054 (0.011)} & 0.191 (0.02)           & 0.211 (0.022)          \\
cycle             & \textbf{0.022 (0.005)} & 0.075 (0.008)         & \textbf{0.023 (0.005)} & 0.043 (0.006)          & \textbf{0.025 (0.005)} \\
electric          & 0.052 (0.007)          & 0.128 (0.009)         & \textbf{0.024 (0.005)} & 0.047 (0.007)          & \textbf{0.029 (0.006)} \\
homes             & 0.135 (0.007)          & \textbf{0.19 (0.011)} & \textbf{0.019 (0.005)} & \textbf{0.016 (0.003)} & 0.038 (0.005)          \\
meps19            & 0.17 (0.012)           & 0.183 (0.015)         & \textbf{0.034 (0.008)} & \textbf{0.022 (0.006)} & 0.053 (0.016)          \\
protein           & 0.063 (0.003)          & 0.071 (0.005)         & 0.062 (0.003)          & \textbf{0.013 (0.003)} & 0.043 (0.005)          \\
star              & 0.076 (0.012)          & \textbf{0.037 (0.01)} & 0.073 (0.01)           & 0.156 (0.012)          & 0.335 (0.016)          \\
superconductivity & 0.072 (0.006)          & 0.254 (0.005)         & \textbf{0.016 (0.004)} & \textbf{0.019 (0.004)} & \textbf{0.025 (0.006)} \\
WEC &  \textbf{0.012 (0.002)} & 0.115 (0.004) & 0.21 (0.007) & \textbf{0.009 (0.002)} & 0.059 (0.007)\\
winered           & \textbf{0.05 (0.01)}   & 0.119 (0.016)         & \textbf{0.042 (0.009)} & 0.153 (0.018)          & 0.221 (0.019)          \\
winewhite         & \textbf{0.035 (0.007)} & 0.079 (0.011)         & \textbf{0.025 (0.005)} & 0.055 (0.009)          & 0.092 (0.011)          \\ \hline
\end{tabular}
\end{adjustbox}
\end{table}

\begin{table}[h]
\caption{Average outlier-to-inlier interval length ratio across methods and datasets in the regression setting. The table reports the average ratio between prediction interval lengths for outliers and inliers, averaged over 50 runs. Bold values indicate the best-performing method within a 95\% confidence interval. The interval length ration for Regression Split is omitted, as its constant interval length results will always return a ratio of 1. Overall, \ourmethod{} consistently yields higher ratios across diverse datasets.}
\label{tab:interval_ratio_reg}
\centering
\begin{adjustbox}{max width = 0.85\textwidth}
\begin{tabular}{lccccc}
\hline
\textbf{Dataset}  & \textbf{EPIC-BART}     & \textbf{EPIC-GP}       & \textbf{EPIC-MDN}      & \textbf{Mondrian}      & \textbf{Weighted}      \\ \hline
airfoil           & \textbf{1.021 (0.039)} & \textbf{1.004 (0.014)} & \textbf{1.03 (0.035)}  & \textbf{0.988 (0.023)} & \textbf{1.03 (0.045)}  \\
bike              & 0.971 (0.014)          & 0.98 (0.006)           & 0.963 (0.02)           & \textbf{1.029 (0.018)} & 0.976 (0.016)          \\
concrete          & \textbf{1.05 (0.036)}  & 1.03 (0.018)           & \textbf{1.11 (0.035)}  & 0.99 (0.036)           & \textbf{1.137 (0.067)} \\
cycle             & 0.967 (0.007)          & 0.992 (0.002)          & 0.973 (0.008)          & \textbf{1.005 (0.006)} & 0.988 (0.005)          \\
electric          & 0.987 (0.008)          & \textbf{1.0 (0.003)}   & 0.992 (0.01)           & \textbf{0.996 (0.005)} & \textbf{0.992 (0.008)} \\
homes             & 1.124 (0.019)          & 1.138 (0.015)          & \textbf{1.268 (0.035)} & \textbf{1.234 (0.026)} & \textbf{1.203 (0.035)} \\
meps19            & \textbf{1.047 (0.027)} & \textbf{1.057 (0.041)} & \textbf{1.097 (0.065)} & \textbf{1.099 (0.05)}  & \textbf{1.074 (0.058)} \\
protein           & 0.999 (0.002)          & 0.996 (0.003)          & \textbf{1.008 (0.004)} & \textbf{1.003 (0.004)} & \textbf{1.0 (0.005)}   \\
star              & \textbf{1.0 (0.007)}   & 1.0 (0.001)            & \textbf{1.005 (0.007)} & \textbf{1.016 (0.011)} & 0.984 (0.025)          \\
superconductivity & 1.129 (0.014)          & 1.061 (0.009)          & 1.107 (0.016)          & \textbf{1.211 (0.014)} & 1.111 (0.017)          \\
WEC               & 1.077 (0.007)          & 1.034 (0.005)          & 1.122 (0.015)          & \textbf{1.2 (0.014)}   & 1.041 (0.007)          \\
winered           & 1.055 (0.019)          & 1.031 (0.009)          & \textbf{1.143 (0.027)} & 1.052 (0.02)           & 1.095 (0.045)          \\
winewhite         & 1.01 (0.007)           & 1.001 (0.001)          & \textbf{1.039 (0.012)} & \textbf{1.058 (0.013)} & 0.997 (0.011)          \\ \hline
\end{tabular}
\end{adjustbox}
\end{table}

\begin{table}[h]
\caption{Average coverage on outliers across methods and datasets in the regression setting. This table reports the
average prediction interval coverage for outlier observations—identified as data-sparse points—across 50 runs. Coverage
values closer to the nominal level of 0.9 reflect better adaptability to epistemic uncertainty. Bold values indicate the
best-performing method within a 95\% confidence interval. \ourmethod{} consistently delivers near-nominal coverage across
most datasets.}
\label{tab:coverage_outlier_reg}
\centering
\begin{adjustbox}{max width = 0.925\textwidth}
\begin{tabular}{lcccccc}
\hline
\textbf{Dataset}  & \textbf{EPIC-BART}     & \textbf{EPIC-GP}       & \textbf{EPIC-MDN}      & \textbf{Mondrian}      & \textbf{Reg-split}     & \textbf{Weighted}      \\ \hline
airfoil           & \textbf{0.896 (0.027)} & \textbf{0.883 (0.028)} & \textbf{0.894 (0.027)} & \textbf{0.91 (0.025)}  & \textbf{0.902 (0.023)} & \textbf{0.903 (0.021)} \\
bike              & \textbf{0.895 (0.01)}  & \textbf{0.901 (0.011)} & \textbf{0.897 (0.01)}  & \textbf{0.909 (0.009)} & \textbf{0.903 (0.01)}  & \textbf{0.895 (0.009)} \\
concrete          & \textbf{0.895 (0.029)} & \textbf{0.902 (0.027)} & \textbf{0.893 (0.028)} & \textbf{0.904 (0.023)} & \textbf{0.873 (0.033)} & \textbf{0.876 (0.029)} \\
cycle             & \textbf{0.89 (0.009)}  & \textbf{0.901 (0.01)}  & \textbf{0.892 (0.008)} & \textbf{0.914 (0.008)} & \textbf{0.901 (0.009)} & \textbf{0.898 (0.009)} \\
electric          & \textbf{0.9 (0.008)}   & \textbf{0.898 (0.01)}  & \textbf{0.902 (0.008)} & \textbf{0.909 (0.008)} & \textbf{0.906 (0.009)} & \textbf{0.911 (0.009)} \\
homes             & 0.869 (0.008)          & 0.873 (0.008)          & \textbf{0.896 (0.006)} & \textbf{0.894 (0.007)} & 0.847 (0.008)          & 0.878 (0.008)          \\
meps19            & \textbf{0.899 (0.007)} & \textbf{0.9 (0.007)}   & \textbf{0.899 (0.008)} & \textbf{0.901 (0.012)} & \textbf{0.894 (0.008)} & \textbf{0.898 (0.007)} \\
protein           & \textbf{0.89 (0.005)}  & \textbf{0.892 (0.005)} & \textbf{0.89 (0.005)}  & \textbf{0.898 (0.004)} & \textbf{0.895 (0.004)} & \textbf{0.892 (0.004)} \\
star              & \textbf{0.892 (0.019)} & \textbf{0.899 (0.019)} & \textbf{0.906 (0.018)} & \textbf{0.899 (0.015)} & \textbf{0.885 (0.02)}  & \textbf{0.879 (0.018)} \\
superconductivity & \textbf{0.891 (0.007)} & 0.872 (0.008)          & \textbf{0.887 (0.008)} & 0.918 (0.005)          & \textbf{0.884 (0.007)} & \textbf{0.888 (0.006)} \\
WEC               & 0.861 (0.006)          & 0.85 (0.007)           & 0.871 (0.006)          & \textbf{0.897 (0.005)} & 0.85 (0.006)           & 0.852 (0.006)          \\
winered           & \textbf{0.87 (0.028)}  & \textbf{0.841 (0.029)} & \textbf{0.881 (0.027)} & \textbf{0.858 (0.024)} & \textbf{0.843 (0.031)} & \textbf{0.846 (0.029)} \\
winewhite         & 0.872 (0.015)          & \textbf{0.876 (0.014)} & \textbf{0.878 (0.016)} & \textbf{0.9 (0.011)}   & 0.863 (0.015)          & 0.867 (0.016)          \\ \hline
\end{tabular}
\end{adjustbox}
\end{table}


\subsection{Base model hyperparameters}
\label{sec::base_model_details}
For the CatBoost quantile regression model \citep{dorogush2018catboost}, we set the number of iterations (trees) to 1,000 and the learning rate to 0.001, enabling early stopping after 50 rounds of no improvement. To mitigate overfitting even further, we limit each tree to a maximum depth of 6. All other parameters follow the default CatBoost settings. In the regression setting, we implemented a Neural Network with three hidden layers, consisting of 64, 32, and 16 neurons, respectively. Each layer utilizes ReLU activation, batch normalization \citep{ioffebatch2015}, and dropout rates \citep{srivastava2014dropout} of 0.2, 0.1, and 0.05, correspondingly. We train the model using a smooth L1 loss, as it provides a balance between mean absolute error (MAE) and mean squared error (MSE), making it more robust to outliers while maintaining stable gradient updates.

For optimization, we utilize the Adam optimizer \citep{kingma2014adam} with an initial learning rate of 0.01. A learning rate scheduler is incorporated to decrease the learning rate by a factor of 0.5 if there is no improvement after 10 epochs, accelerating convergence. All weights are initialized using Xavier normal initialization \citep{kumar2017weight}. We set aside 30\% of the training data for validation and set a batch size of $35$. Training proceeds for a maximum of 750 epochs, with early stopping triggered if there is no improvement on the validation set for 30 consecutive epochs. Additionally, feature scaling and target min-max normalization are applied to ensure stable training.

\section{\ourmethod{} Computational and Methodological Details}
\label{sec::comp_details}
\label{}
This section outlines the specifications of each predictive model used in our framework, discusses the scalability of each approach, and details the trade-offs associated with prior choices.

\subsection{Splitting strategy}
When splitting the calibration set $\mathcal{D}_{\text{cal}}$ for fitting the predictive model and deriving \ourmethod{}'s adaptive cutoffs in disjunct data subsets, we prioritize allocating the majority of data to model training, as cutoff derivation primarily involves a simpler quantile computation. Specifically, for small with $n \leq 3000$, we use 30\% of the calibration samples for cutoff computation. For larger datasets, this allocation is capped at 1000 samples to maintain computational efficiency. This approach ensures that \ourmethod{} achieves both accurate predictive distribution estimates and well-calibrated cutoffs.

\subsection{MDN MC-Dropout details}
\label{sec::mc-dropout-architecture}
The employed Mixture Density Network (MDN) consists of $2$ hidden layers, each with 64 neurons, ReLU activations, and batch normalization \citep{ioffebatch2015} to enhance training stability. The output layer predicts three parameters for each mixture component $k = \{1, \dots, K\}$: membership probability $\pi_k(\cdot)$,  the mixture mean $\mu_k(\cdot)$ and the mixture variance $\sigma^2_k(\cdot)$. We set the number of components to $K= 3$, outputting 9 parameters. The network is trained using the negative log-likelihood loss, applying a softmax activation to the mixture probabilities and a softplus activation to the variance estimates. 

Optimization relies on the Adam optimizer \citep{kingma2014adam} with a learning rate of 0.001 and a step learning rate scheduler that decays by 0.99 every 5 epochs. To model epistemic uncertainty, we incorporate MC dropout \citep{gal2016dropout}, applying a dropout rate of 0.5 to each hidden layer. This setup implicitly corresponds to a diffuse, non-informative prior over the network weights—one that does not favor any specific parameter configuration and assumes equal uncertainty over weights. All network weights are initialized using PyTorch’s default scheme. To monitor generalization, 30\% of the training data is held out for validation. The model is trained for up to 2,000 epochs, with early stopping after 50 epochs without validation improvement. Feature scaling and target normalization are also applied to improve numerical stability and parameter estimation.

Batch size selection is dataset-dependent: $40$ for small datasets $(n < 10000)$, $125$ for medium datasets ($n < 50000$), and $250$ for large-scaled datasets, such as WEC. For the image experiment, we introduce an additional hidden layer with $32$ neurons and maintain a dropout rate of $0.5$, while adjusting the batch size to $135$. To compute the predictive CDF $F(s(\X, Y)|\X, D)$ at $\X = \x$, we generate 500 samples of the mixture parameters using MC dropout forward passes. For each sampled mixture parameter set, score samples are drawn from the Gaussian Mixture Model. The final predictive CDF is obtained by computing the empirical distribution of $s(\X,Y)$ over these score samples.

\subsection{Variational GP details}
For the Variational Gaussian Process (GP), we implement the model using the \textit{gpytorch} package \citep{gardner2018gpytorch} and PyTorch. The GP prior is specified with a constant mean function and a Radial Basis Function (RBF) kernel, which provides a smoothness assumption but remains essentially non-informative in the sense that it does not encode strong prior beliefs about the function values or parameters. To approximate the posterior, we employ a Natural Variational Distribution \citep{salimbeni2018natural}, using 15 inducing points for small datasets ($n < 10000$) and $50$ for medium and large-scale data. Training is performed by minimizing the negative variational ELBO, where the Natural Gradient Descent \citep{salimbeni2018natural} updates the variational parameters, while the Adam optimizer \citep{kingma2014adam} refines the GP kernel and noise variance hyperparameters. 

The model is trained for up to 2000 epochs, with early stopping triggered after 50 epochs of no improvement. Following the MDN predictive model,  we reserve 30\% of the training data exclusively for validation and adopt an adaptive batch size of $40$ for small datasets ($n < 10000$), $125$ for medium-sized datasets ($n < 50000$), and $250$ for large datasets. To ensure numerical stability, we apply feature scaling and target normalization. Using the learned variational gaussian posterior, we easily derive the predictive CDF $F(s(\X, Y)|\X, D)$ at $\X = \x$ by using gaussian conjugacy.

\subsection{BART details}
\label{sec::bart_details}
For the Bayesian Additive Regression Tree (BART) model, we use the base implementation from \textit{pymc3} \citep{quiroga2022bayesian} and adopt the heteroscedastic variant \citep{pratola2020heteroscedastic}. The conformal scores are modeled as a normal distribution, where the mean is determined by the sum of regression trees, and the variance depends on $\X$.  We configure the model with 100 trees and retain the default prior settings provided by \textit{pymc3}, which correspond to a non-informative prior favoring deeper trees. This choice allows for flexible modeling without imposing strong structural constraints. Additionally, we apply target normalization to improve numerical stability and enhance posterior inference.

After obtaining BART posterior samples via MCMC \citep{chipman2010bart}, we derive the predictive CDF $F(S(\X, Y)|\X, D)$ for a given $\X = \x$ by simulating scores from the postulated distribution using the posterior samples. The corresponding empirical CDF for the score $S(\X, Y)$ is then computed. In the regression example presented in Figure \ref{fig::reg_split} and detailed in Appendix \ref{sec::technicalReg}, we used a modified version of BART. In this version, a heteroscedastic gamma distribution was assigned to the score instead of a normal distribution, which was necessary to account for the asymmetry in the regression conformal scores for that example.

\subsection{Impact of Prior Choices}
\label{sec::prior_comparisson}
In this section, we illustrate the impact of prior choices in \ourmethod{} through a simulated experiment comparing the prediction intervals produced by \ourmethod{}-BART under two prior configurations: \textit{diffuse} and \textit{concentrated}. To emulate data sparsity, we follow a similar setup to Section \ref{sec::technicalReg}, simulating a scenario with both high- and low-density regions. Specifically, given a sample size $n$, we draw $\left \lfloor{0.8 \cdot n}\right \rfloor$ samples from $X \sim U(-1, 0)$, representing the denser region. The remaining $\left \lfloor{0.2 \cdot n}\right \rfloor$ samples are generated from $X \sim \text{Beta}(5, 5)$, concentrating points in the intermediate region of the interval [0,1]. The response variable follows a heteroscedastic distribution $Y \sim N\left(\I\{x > 0\}, 0.05 + (\sin^2(15 \cdot x) \cdot \I\{x > 0\})\right)$, which results in low aleatoric and epistemic uncertainty over the densely sampled interval $[-1, 0]$ and high uncertainty—both epistemic and aleatoric—over the sparse interval $[0,1]$.

The dataset is divided into training, calibration, and test sets, each containing 200 samples. As the base predictor, we use a k-nearest neighbors (k-NN) regression model with  $k = 10$, trained on the training set. For \ourmethod{}, both BART configurations share the same settings: $m = 50$ trees, heteroscedastic variance, a normal likelihood, response normalization, and $1000$ Monte Carlo samples. The key difference lies in the prior hyperparameter $\beta$,which governs tree depth by controlling the probability that a node is non-terminal. A higher value, such as $\beta = 0.9$ yields a diffuse prior with deeper trees and greater model flexibility, while a lower value like $\beta = 0.1$ results in a concentrated prior with shallower trees, imposing stronger regularization. 

Figure \ref{fig:difused_versus_concentrated_priors} illustrates how the prediction intervals behave under each prior configuration.
\begin{figure}[h]
    \centering
    \includegraphics[width=0.625\linewidth]{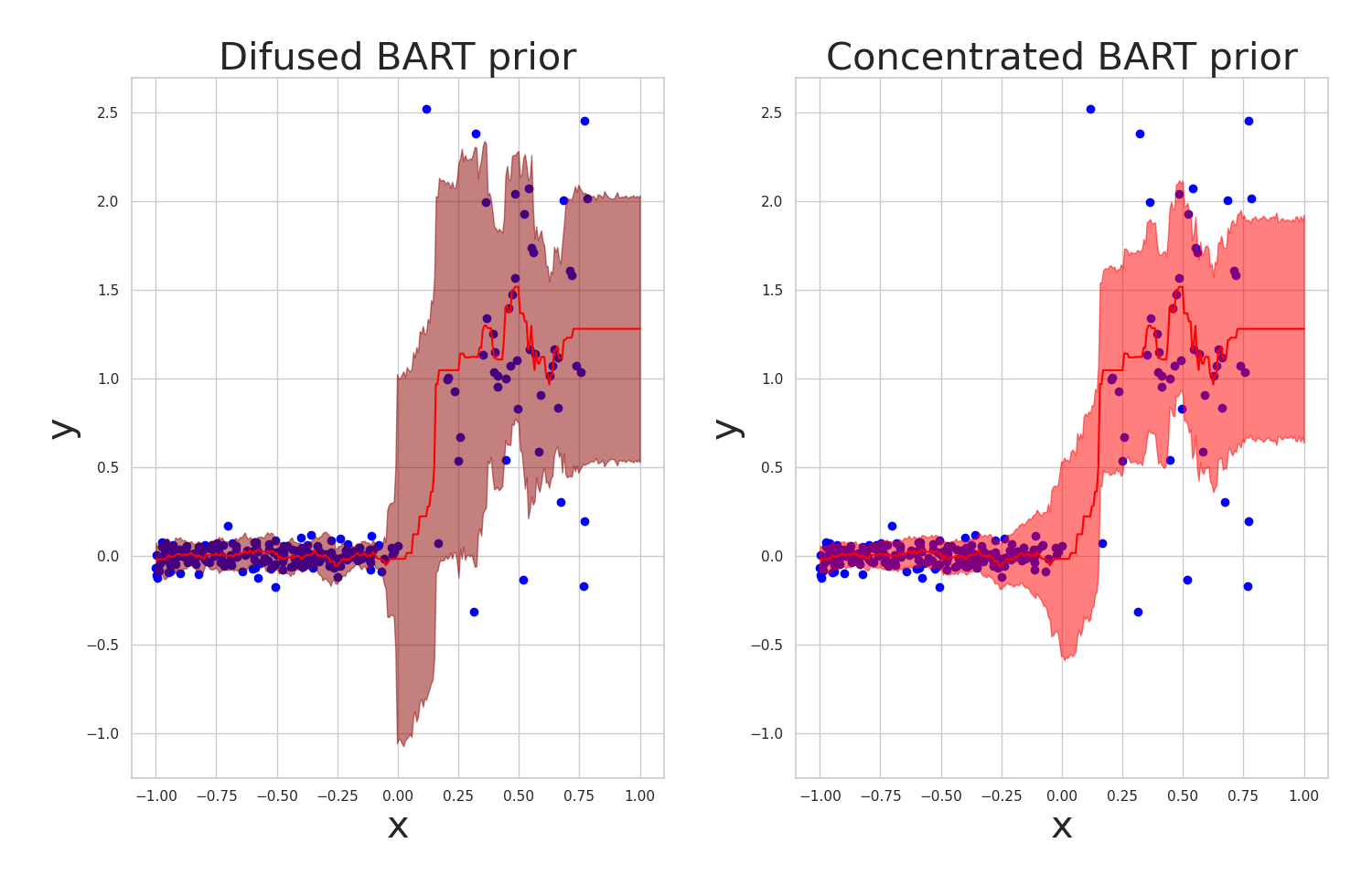}
    \caption{
    Comparison of prediction regions produced by \ourmethod{} using BART under two prior configurations: a diffuse prior and a concentrated prior. While the diffuse prior produces wider prediction regions—particularly in sparse areas—due to greater model flexibility, the concentrated prior yields narrower intervals reflecting stronger regularization. This highlights the sensitivity of \ourmethod{} to prior specification.
    }
    \label{fig:difused_versus_concentrated_priors}
\end{figure}
As shown, the diffuse prior yields wider intervals—particularly in data-sparse regions—capturing heightened epistemic uncertainty. In contrast, the concentrated prior produces tighter intervals in denser regions, reflecting its stronger regularization and more confident assumptions in well-populated areas. This highlights the sensitivity of \ourmethod{} to prior specification and demonstrates how prior concentration directly influences the adaptiveness and sharpness of the prediction sets.

To further investigate this trade-off, we also evaluate the influence of $\beta$ using the AISL metric, which jointly accounts for coverage and interval length. As shown in Figure \ref{fig:AISL_versus_alpha}, higher values of $\beta$ lead to deeper trees and wider intervals, while smaller values yield more concise intervals due to stronger regularization.
\begin{figure}[h]
    \centering
    \includegraphics[width=0.675\linewidth]{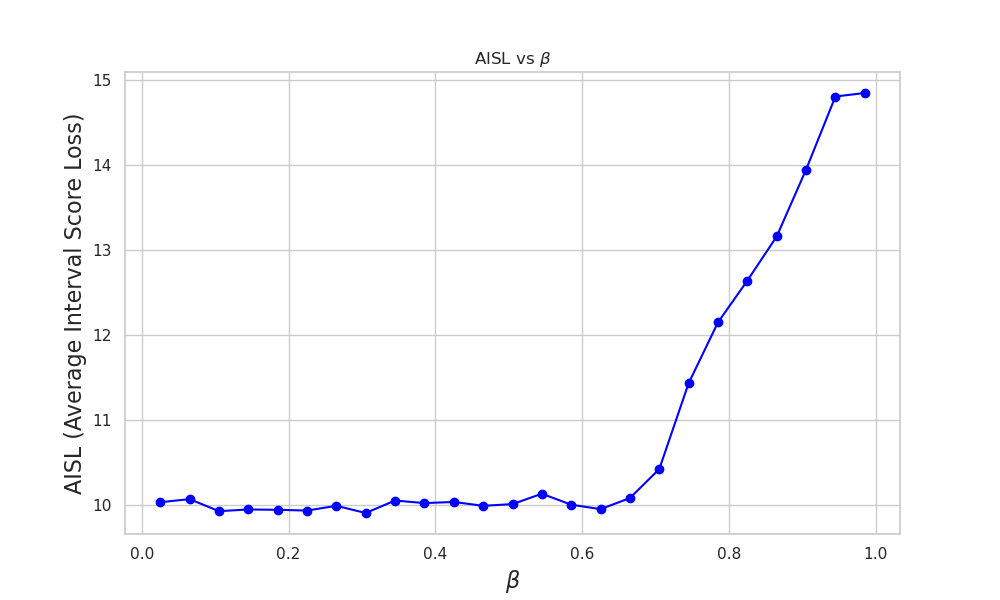}
    \caption{Variation of the \textbf{Average Interval Score Loss} (AISL) as a function of the $\beta$ hyperparameter in the BART model used within the \ourmethod{} framework. Lower AISL values indicate more favorable balance, suggesting that tuning $\beta$ can significantly affect performance by controlling model flexibility and prior regularization.}
    \label{fig:AISL_versus_alpha}
\end{figure}

This analysis helps reveal the trade-off between interval regularization and accurate representation of epistemic uncertainty. Crucially, it also provides a foundation for tuning prior hyperparameters more systematically. By optimizing AISL—or another relevant performance metric—one can select prior configurations in a principled manner. In more complex scenarios involving multiple hyperparameters, this process remains scalable through the use of Bayesian optimization techniques \citep{frazier2018bayesian}.

\subsection{Scalability comparisson}
\label{sec::computational_complexity}
Incorporating Bayesian models into conformal prediction offers a principled way to capture epistemic uncertainty but may introduce additional computational overhead. To quantify this cost within our framework, we evaluate the runtime of \ourmethod{} across increasing sample sizes using the three Bayesian predictive models discussed previously: MDN, BART, and Variational GP. We consider a synthetic regression task with 30 input features and one response variable, generated using scikit-learn’s \texttt{make\_regression} function \citep{pedregosa2011scikit}. For each sample size  (\(n = 1000, 2000, 5000, 10000, 20000, 50000\)), we generate training and calibration sets of equal size ($n$), train a k-nearest neighbors (k-NN) model with $k = 30$ as the base predictor on the training data, and fit the predictive models on the calibration data. Each setting is repeated 10 times, and Figure~\ref{fig:running_time_versus_n} reports the average runtime along with its standard deviation for each sample size.
\begin{figure}[h]
    \centering
    \includegraphics[width=0.65\linewidth]{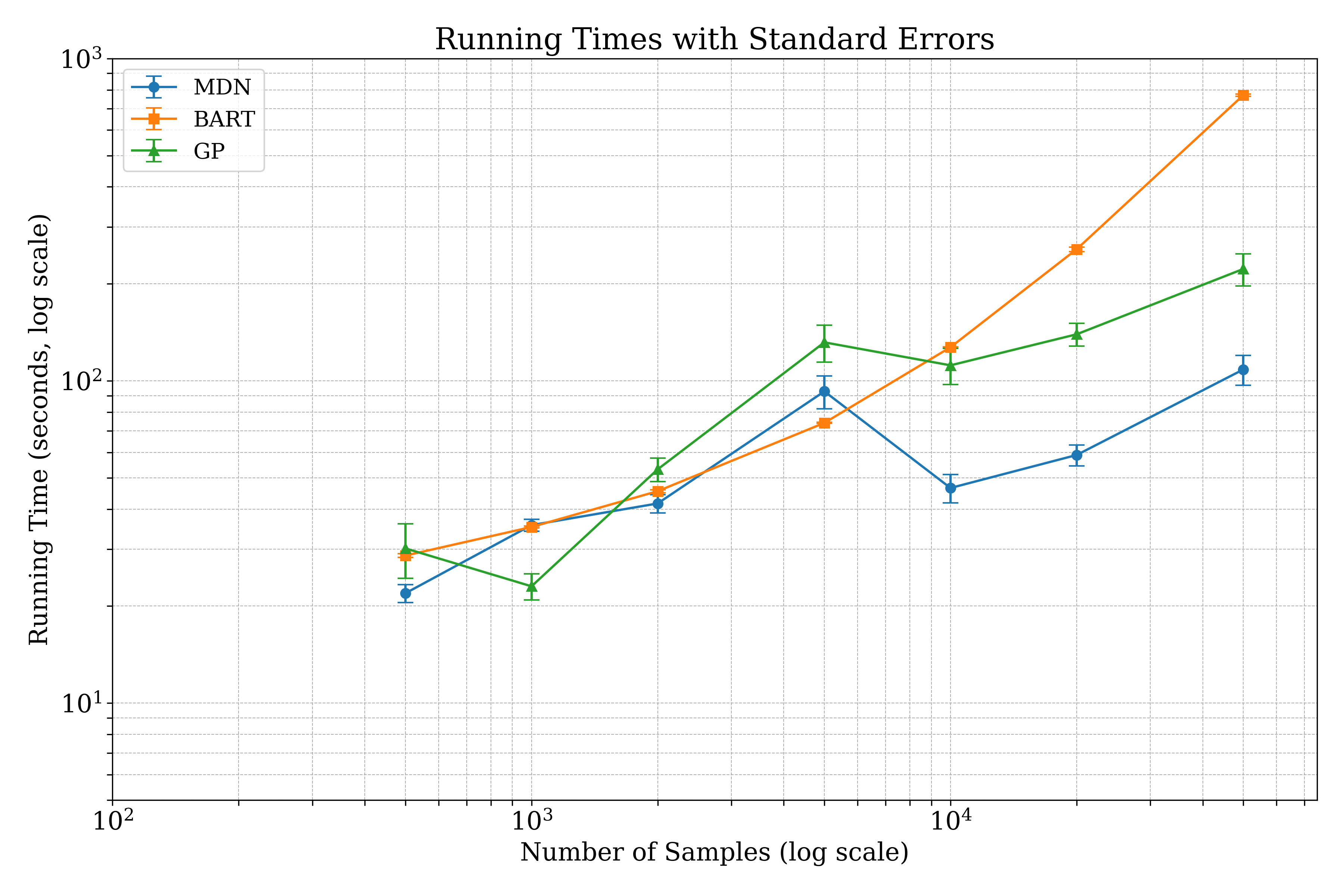}
    \caption{Computational cost of \ourmethod{} as a function of sample size $n$ (log scale), using the three Bayesian predictive models: BART, MDN and Variational GP. While all three models show comparable running times at moderate scales(e.g., $n = 10000$),  MDNs exhibit significantly better scalability and faster execution as $n$ increases, underscoring their advantage in large-scale settings.}
    \label{fig:running_time_versus_n}
\end{figure}

The MDN configuration includes two hidden layers (64–64 units), ReLU activations, 50\% dropout, $2000$ training epochs, a learning rate of $0.001$, and batch size dynamically scaled to the sample size (ranging from $35$ to $450$). Inputs are standardized, and the response is normalized. BART is configured with $50$ trees, normal heteroscedastic errors, $500$ posterior samples, and parallelization using $6$ CPU cores. The GP model is trained via a variational approximation, with the number of inducing points dynamically adjusted from $15$ to $150$ based on \(n\), a variational learning rate of $0.1$, a hyperparameter learning rate of $0.01$, and identical batch size and early stopping parameters to the MDN.

Overall, MDNs demonstrate the most stable and scalable runtime across all sample sizes. Variational Gaussian Processes also scale well, though they are generally less efficient than MDNs. The observed drop in runtime between $5000$ and $10000$ samples for both neural network–based models likely stems from the early stopping mechanism, which halts training once convergence is detected. BART, on the other hand, maintains low computational cost for small datasets but exhibits a sharp increase in runtime as the sample size grows. This highlights a key trade-off: while BART is effective for small to medium-scale problems, neural network–based models like MDNs are better suited for larger datasets, where the iterative nature of MCMC-based methods such as BART becomes increasingly expensive.

This experiment was conducted using an AMD Ryzen 7 6000 series CPU. For the largest dataset, which includes over $100000$ calibration observations, the runtime for each method remained small: BART completed in under 16 minutes, MDN in under 2 minutes, and GP in under 4 minutes. These results highlight the computational efficiency of the proposed approaches, particularly the scalability of MDNs and GPs for large-scale applications.

\subsection{Application to density-based conformity scores}
\label{sec::dens_epicscore}
In this section, we use a simulated experiment to illustrate how \ourmethod{} can enhance density-based conformal prediction methods by explicitly incorporating epistemic uncertainty. Following the same simulation setup described in Section~\ref{sec::prior_comparisson}, we compare our approach to a standard method for conformal prediction with a density-based score, namely HPD-split \citep{izbicki2022cd}. The experiment is designed to reflect regions with both high epistemic and aleatoric uncertainty, allowing us to highlight the advantages of our method in capturing complex uncertainty structures.

The base conditional density model is a Mixture Density Network (MDN) consisting of two hidden layers with 64 and 128 neurons, respectively, and a Gaussian mixture output with 5 components. The model incorporates a dropout rate of 0.15 to improve generalization and is trained on 2000 samples. For \ourmethod{}, we use a Bayesian Additive Regression Trees (BART) model with 50 trees, a gamma error distribution, heteroscedastic variance, and a non-informative prior. The tree depth is controlled by a fixed hyperparameter $\beta = 0.9$, promoting deeper trees. Calibration is performed using 1000 posterior samples. The HPD-split prediction region is constructed directly from the MDN’s conditional density estimates. Figure~\ref{fig:HPD_versus_epicscore} compares the prediction regions obtained by the two methods.
\begin{figure}[h]
    \centering
    \includegraphics[width=1\linewidth]{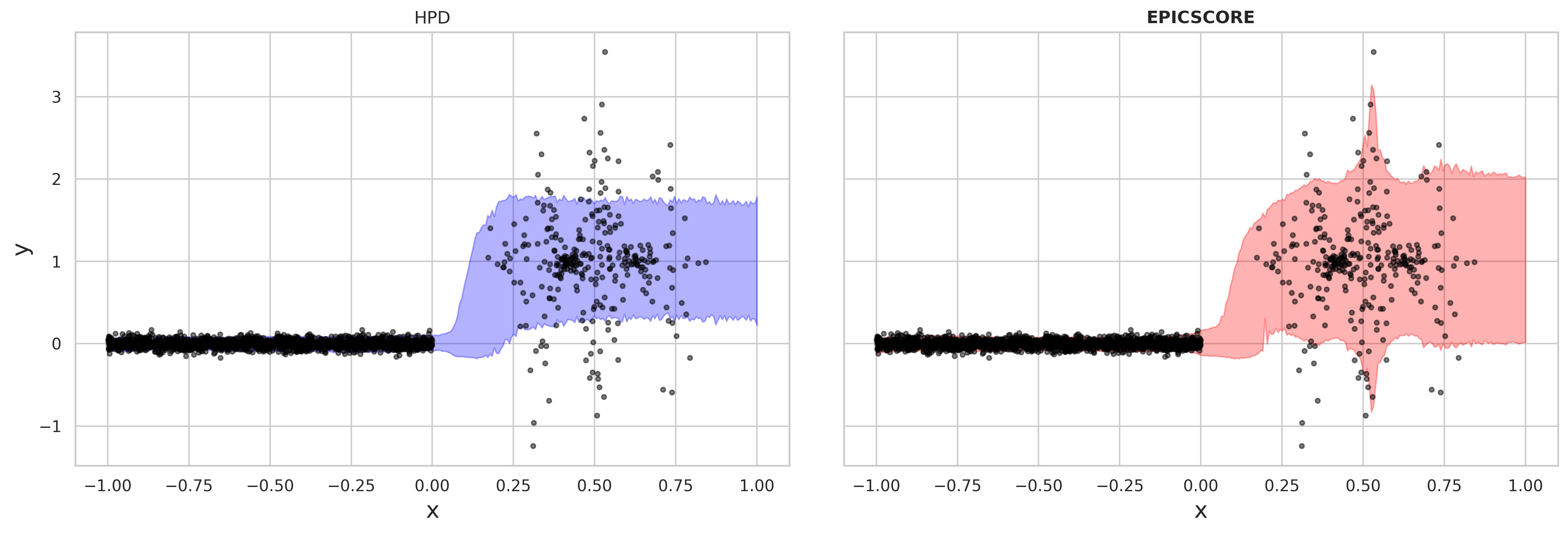}
    \caption{Comparison of prediction regions generated by the HPD (Highest Predictive Density) split \citep{izbicki2022cd} and \ourmethod{} methods in a simulated scenario. Both methods produce similar regions in the high-density area (\(x < 0\)), but \ourmethod{} exhibits sharper adaptivity in the sparse region (\(x > 0\)), offering more informative and expanded predictive sets where epistemic uncertainty is higher.}
    \label{fig:HPD_versus_epicscore}
\end{figure}

Compared to the HPD-based region, our method exhibits significantly better adaptiveness and uncertainty quantification, particularly in data-sparse regions of the input space. At the same time, it preserves tight prediction intervals in denser, data-rich areas. This demonstrates \ourmethod{}’s ability to more effectively capture uncertainty in challenging settings where data is limited, without sacrificing performance in well-populated regions.

\section{Proofs}
\label{sec:proofs}

\begin{proof}[Proof of Theorem \ref{thm::marginal}]
Follows immediately from the fact that  \ourmethod\ is a conformal score and \citet[Theorem 2]{Lei2018}.
\end{proof}


\begin{proof}[Proof of Theorem \ref{thm::conditional_coverage}]

Let $t'_{1-\alpha}$ be the empirical quantile obtained using $s'(\x,y) = F(s(\x,y)|\x,D)$, and let $t''_{1-\alpha}$ be the empirical quantile obtained using $s''(\x,y) = F(s(\x,y)|\x,\theta^*)$. 

By Assumption \ref{assumption:uniform_convergence}, for any $\varepsilon,\delta > 0$, if the calibration set is sufficiently large, then with probability at least $1 - \delta$ (for some event $\Omega$), we have
\[
\sup_{s, \x} \left|F(s \mid \x, D) - F(s \mid \x, \theta^*)\right| \leq \varepsilon,
\]
where the randomness is over $D$.

Let $\hat{\P}$ denote the empirical probability measure based on $\mathcal{D}_{\calib,2}$, i.e., for a given event \( A \) and any function $g$:
\[
\hat{\P}(g(s) \in A) = \frac{1}{|\mathcal{D}_{\calib,2}|} \sum_{(\X_i, Y_i) \in \mathcal{D}_{\calib,2}} \Ind(g(s(\X_i, Y_i)) \in A).
\]

Conditionally on the event $\Omega$, we obtain:
\begin{align*}
   1 - \alpha &\leq \hat{\P}(F(s | \x, D) \leq t'_{1-\alpha})\\
   &\leq \hat{\P}(F(s | \x, \theta^*) - \varepsilon \leq t'_{1-\alpha}),
\end{align*}
which, by the definition of the empirical quantile, implies that 
\[
t''_{1-\alpha} \leq t'_{1-\alpha} + \varepsilon.
\]
By a symmetric argument, we also have:
\begin{align*}
   1 - \alpha &\leq \hat{\P}(F(s | \x, \theta^*) \leq t''_{1-\alpha})\\
   &\leq \hat{\P}(F(s | \x, D) - \varepsilon \leq t''_{1-\alpha}).
\end{align*}
Therefore, conditionally to the event $\Omega$, of probability at least $1 - \delta$, we conclude that
\[
|t'_{1-\alpha} - t''_{1-\alpha}| \leq \varepsilon.
\]

Using this result and the fact that $\left|F(s \mid \x, D) - F(s \mid \x, \theta^*)\right| \leq \varepsilon$, we establish the following bound:
\begin{align*}
    \P(s'(\X,Y) \leq t'_{1-\alpha}|\X) &\leq \P(s'(\X,Y) \leq t'_{1-\alpha} \mid \X,\Omega) \cdot 1 + \delta\\
    & \leq \P(s'(\X,Y) \leq t''_{1-\alpha} + \varepsilon \mid \X,\Omega) + \delta\\
    &\leq \P(s''(\X,Y) - \varepsilon \leq t''_{1-\alpha} + \varepsilon|\X,\Omega) + \delta\\
    &=  \P(s''(\X,Y) \leq t''_{1-\alpha} + 2\varepsilon \mid \X) + \delta\\
    &= t''_{1-\alpha} + 2\varepsilon + \delta.
\end{align*}
In the fourth equality, we used the fact that \( s''(x,y) \) is non-random and that \( t''_{1-\alpha} \) depends only on \( \mathcal{D}_{\calib, 2} \), making it independent of \( \Omega \). In the last equality, we used the fact that the random variable $s''(\X,Y)|\X$ is uniform.

By a similar argument, we can show that:
\[
\left|\P(s'(\X,Y) \leq t'_{1-\alpha}|\X) - t''_{1-\alpha}\right| \leq 2\varepsilon + \delta.
\]
Thus, as $|\mathcal{D}_{\calib,1}| \to \infty$, we can take $\varepsilon, \delta \to 0$, which implies that
\[
\lim_{|\mathcal{D}_{\calib,1}| \to \infty} \P(s'(\X,Y) \leq t'_{1-\alpha}|\X) = t''_{1-\alpha}.
\]
Finally, by \cite[Lemma 2, Section D.2.2]{dheur2025multi}, we know that as $|\mathcal{D}_{\calib,2}| \to \infty$, we have $t''_{1-\alpha} \to 1-\alpha$, which concludes the proof.

\end{proof}


\end{document}